\documentclass[twoside,11pt]{article}

\usepackage[nohyperref]{jmlr2e}
\usepackage{amssymb}
\usepackage{natbib}
\usepackage{color}
\usepackage{bbm}
\usepackage[hidelinks,breaklinks=true]{hyperref}
\usepackage{amsfonts}       
\usepackage{url}
\usepackage{amsmath}
\usepackage{wrapfig}
\usepackage{mathtools}
\usepackage{tikz}
\usepackage{algorithmic}
\usepackage{footmisc}
\usepackage{bibentry}
\usepackage{thmtools, thm-restate}
\nobibliography*

\usepackage{physics}
\usepackage{mathrsfs}
\mathtoolsset{showonlyrefs}
\usepackage{etoolbox}
\usepackage{makecell}
\usepackage[ruled]{algorithm2e}
\DontPrintSemicolon
\usepackage{enumerate}
\usepackage{booktabs}
\usepackage{pifont}
\usepackage{soul}
\usepackage[normalem]{ulem}
\usepackage{cancel}
\usepackage{microtype}

\newcommand{\fS}{\mathcal{S}}
\newcommand{\fA}{\mathcal{A}}

\newcommand{\fY}{\mathcal{Y}}

\newcommand{\fBB}{\mathcal{B}}

\newcommand{\fF}{\mathcal{F}}
\newcommand{\fO}{\mathcal{O}}

\newcommand{\R}[1][]{\mathbb{R}^{#1}}
\newcommand{\indot}[2]{{\left<#1, #2\right>}}
\newcommand{\E}{\mathbb{E}}

\newcommand{\ns}{{\abs{\fS}}}
\newcommand{\na}{{\abs{\fA}}}

\newcommand{\tref}[1]{\text{\ref{#1}}}
\newcommand{\explain}[1]{\tag*{(#1)}}
\newcommand{\cmark}{\ding{51}}%

\newcommand{\inner}[2]{\langle #1, #2 \rangle}

\newcommand{\tb}[1]{{\textbf{#1}}}

\usepackage{pdfrender}
\newcommand{\dashcheckmark}{
    \textpdfrender{
        TextRenderingMode=Stroke,
        LineWidth=0.5pt,
        LineDashPattern=[1 1]0,
    }{\checkmark}
}

\allowdisplaybreaks

\newcounter{assucounter}
\numberwithin{assucounter}{section}
\newtheorem{assumption}[assucounter]{Assumption}
\newtheorem{assumptionalt}{Assumption}[assumption]
\newenvironment{assumptionp}[1]{
  \renewcommand\theassumptionalt{#1}
  \assumptionalt
}{\endassumptionalt}

\usepackage{lastpage}
\jmlrheading{27}{2026}{1-\pageref{LastPage}}{6/25; Revised 5/26}{11/26}{25-2392}{Xinyu Liu, Zixuan Xie, Shangtong Zhang}
\ShortHeadings{Extensions of Robbins-Siegmund Theorem}{Liu, Xie, and Zhang}

\firstpageno{1}

\begin{document}
\title{Extensions of Robbins-Siegmund Theorem\\ with Applications in Reinforcement Learning}
\author{%
\name Xinyu Liu \email xinyuliu@virginia.edu\\
  \addr Department of Computer Science\\
  \addr University of Virginia\\
  85 Engineer's Way, Charlottesville, VA, 22903, United States
\AND
\name Zixuan Xie \email xie.zixuan@email.virginia.edu\\
  \addr Department of Computer Science\\
  \addr University of Virginia\\
  85 Engineer's Way, Charlottesville, VA, 22903, United States
\AND
\name Shangtong Zhang \email shangtong@virginia.edu \\
  \addr Department of Computer Science\\
  \addr University of Virginia\\
  85 Engineer's Way, Charlottesville, VA, 22903, United States
  }

\editor{Nan Jiang}

\maketitle
\begin{abstract}
The Robbins-Siegmund theorem establishes the convergence of stochastic processes that are almost supermartingales and is one of the most commonly used approaches for analyzing stochastic iterative algorithms in stochastic approximation and reinforcement learning (RL).
However, its original form has a significant limitation as it requires the zero-order term to be summable.
In many important RL applications,
this summable condition, however, cannot be met.
This limitation motivates us to extend the Robbins-Siegmund theorem for almost supermartingales where the zero-order term is not summable, but only square-summable.
In particular,
we introduce a novel and mild assumption on the increments of the stochastic processes.
This together with the square-summable condition enables an almost sure convergence to a bounded set.
Additionally,
we further provide almost sure convergence rates, high probability concentration bounds, and $L^p$ convergence rates.
We then apply the new results to stochastic approximation and RL.
Notably,
we obtain the first almost sure convergence rate,
the first high probability concentration bound,
and the first $L^p$ convergence rate for $Q$-learning with linear function approximation.

\let\svthefootnote\thefootnote
  \let\thefootnote\svthefootnote
\end{abstract}


\section{Introduction}
\label{sec:intro}

The supermartingale is an important concept in stochastic processes and the celebrated Doob's supermartingale convergence theorem \citep{doob1953stochastic} establishes the almost sure convergence of supermartingales and is arguably one of the most important results in the field of martingales.
However,
many stochastic processes of interest are not supermartingales and, as a result, Doob's theorem does not apply.
Instead,
they are only very close to supermartingales and admit some noisy perturbations from supermartingales.
\citet{robbins1971convergence} call them almost supermartingales.
The Robbins-Siegmund theorem establishes the almost sure convergence of almost supermartingales under mild conditions and can be used to analyze a wide range of reinforcement learning (RL) algorithms whose associated ODE is globally asymptotically stable (GAS).
Examples of RL algorithms with GAS ODEs include
linear temporal difference learning \citep{sutton1988learning} and tabular $Q$-learning \citep{watkins1989learning}.
Numerous methods have been developed to analyze those algorithms and some are based on the Robbins-Siegmund theorem, see, e.g., \citet{qian2026almost}.
However,
there is a class of RL algorithms whose associated ODE is not GAS,
e.g.,
linear SARSA \citep{rummery1994line} and linear $Q$-learning \citep{baird1995residual}.
Linear SARSA is known to chatter inside a ball instead of converging to a single point,
regardless of the choice of the learning rate \citep{gordon1996chattering,gordon2001reinforcement}.
Linear $Q$-learning in general can diverge \citep{baird1995residual}.
But with proper choices of the behavior policy and learning rate,
\citet{meyn2024bellmen} shows that it can remain bounded almost surely.
The follow-up work \citet{liu2025linearq} further shows that it can remain bounded in $L^2$.
Such boundedness is usually the best one can hope for without imposing strong assumptions on the problem structure \citep{meyn2024bellmen}.
Although the Robbins-Siegmund theorem has enjoyed great success in RL algorithms with GAS ODEs,
it falls short in analyzing RL algorithms with non-GAS ODEs.
We argue that the root cause is that the standard Robbins-Siegmund theorem requires its zeroth-order to be summable.
But for many RL algorithms with non-GAS ODEs such as linear SARSA and linear $Q$-learning,
the resulting zeroth-order term is not summable and only square-summable.
This calls for extensions of the Robbins-Siegmund theorem and this paper makes progress towards this goal with the following two contributions.
\begin{enumerate}
  \item[(i)] We introduce a novel and mild assumption on the increments of almost supermartingales.
  This together with a square-summable condition on the zeroth-order term
  ensures
  an almost sure convergence to a bounded set for almost supermartingales,
  even if the zeroth-order term is not summable.
  Additionally,
  we further provide almost sure convergence rates, high probability concentration bounds, and $L^p$ convergence rates,
  all to a bounded set.
  \item[(ii)] Building on our results concerning almost supermartingales,
  we deliver new convergence results of general stochastic approximations with time-inhomogeneous Markovian noise,
  which is then used to analyze linear $Q$-learning,
  yielding the first almost sure convergence rate,
  the first high probability concentration bound,
  and the first $L^p$ convergence rate,
  all to a bounded set.
\end{enumerate}

To put (i) in the literature,
we remark that recent notable extensions of the Robbins-Siegmund theorem include 
\citet{liu2022almost,karandikar2024convergence},
which 
derive almost sure convergence rates for the Robbins-Siegmund theorem with a summable zeroth-order term.
However,
to our knowledge,
no existing work is able to relax the summable condition on the zeroth-order term.

To put (ii) in the literature,
we discuss RL algorithms with GAS and non-GAS ODEs separately.
There are numerous works on RL algorithms with GAS ODEs (e.g., \citet{tsitsiklis1997analysis,abounadi2001learning,sutton2009convergent,zhang2021average,zhang2022truncated,qian2025revisiting,blaser2026almost}).
Most related to this work is \citet{lauand2024revisiting,borkar2025ode}.
In particular,
\citet{borkar2025ode} establish almost sure convergence and $L^4$ boundedness for RL algorithms with GAS ODEs (without a rate).
\citet{lauand2024revisiting} extend \citet{borkar2025ode} by broadening the choice of learning rates and additionally establishing $L^p$ convergence rates.
For almost sure convergence, 
\citet{lauand2024revisiting} still do not provide a rate.
\citet{borkar2025ode} also establish new results about central limit theorem which are not related to this work.
So for the comparison axes relevant here, we focus on
\citet{lauand2024revisiting} and will discuss \citet{borkar2025ode} only sparingly.
Nevertheless,
\citet{lauand2024revisiting} is still not directly comparable to this work since it only applies to RL algorithms with GAS ODEs while this work focuses on RL algorithms with non-GAS ODEs.
Numerous works also consider RL algorithms with non-GAS ODEs,
e.g., policy gradient methods \citep{sutton2000policy}, average-reward RL algorithms with global-clock learning rates \citep{antrobius2026divergence}, and linear RL algorithms without full rank features \citep{xie2025finite,wang2024almost}.
Among those works, 
only a few are related to (ii) and the most notable ones are \citet{meyn2024bellmen} and \citet{liu2025linearq} that study the boundedness of linear $Q$-learning in an almost sure sense and $L^2$ sense respectively.
Arguably,
one may be able to combine the techniques of \citet{lauand2024revisiting} and \citet{meyn2024bellmen}.
Such a hypothetical combination could potentially be used to analyze linear $Q$-learning to further provide an $L^p$ convergence rate beyond the almost sure boundedness result of \citet{meyn2024bellmen}. 
We, however, remark that such a combination has not been carried out.
Furthermore, even if this combination were developed successfully,
it would only affect our claim about ``the first $L^p$ convergence rate'' in  (ii).
It is unclear whether such a combination would yield almost sure convergence rates and nonasymptotic high probability concentration bounds with sub-Weibull tails.
Moreover, our results are applicable to not only linear $Q$-learning but also linear SARSA,
similar to how \citet{meyn2024bellmen} applies to both linear $Q$-learning and linear SARSA.
We focus only on linear $Q$-learning in our presentation to avoid verbosity.

\section{Background}
\label{sec:bg}
This section reviews the classical Robbins-Siegmund theorem and some latest extensions.
We consider a probability space $(\Omega, \mathcal{F}, \Pr)$ equipped with a filtration $\qty{\mathcal{F}_n}_{n \geq 0}$, where $\mathcal{F}_n$ represents the information available up to time $n$. For any random variable $z$, we denote the conditional expectation with respect to $\mathcal{F}_n$ by $\E_n\qty[z] \doteq \E\qty[z |\mathcal{F}_n]$.
The Robbins-Siegmund theorem is concerned with the convergence of a non-negative real-valued random sequence $\qty{z_n}_{n\geq 0}$ satisfying 
  \begin{align}
    \label{eq:rs}
    \tag{RS}
    \textstyle\E_n\qty[z_{n+1}] \leq (1 + a_n)z_n + x_n - y_n  \qq{a.s.,}
  \end{align}
where $\qty{z_n}_{n \geq 0}$, $\qty{a_n}_{n \geq 0}$, $\qty{x_n}_{n \geq 0}$, and $\qty{y_n}_{n \geq 0}$ are \textbf{non-negative} real-valued random sequences adapted to $\qty{\mathcal{F}_n}_{n \geq 0}$. 
This non-negative condition is always imposed in the rest of the paper.
\citet{robbins1971convergence} call such $\qty{z_n}$ almost supermartingales.
Indeed,
if $a_n = x_n = 0$,
then $\qty{z_n}$ immediately becomes a supermartingale.\footnote{This additionally requires verifying that $\qty{z_n}$ is bounded in $L_1$, which is an easy exercise given~\eqref{eq:rs} and the non-negativity of $\qty{z_n}$.}
We now present the Robbins-Siegmund theorem.

\begin{lemma}[Theorem~1 of \citealt{robbins1971convergence}]
  \label{thm:rs}
  Consider $\qty{z_n}$ in~\eqref{eq:rs}.
Assume
\begin{enumerate}
    \item[(A1)] $\sum_{n=0}^{\infty} a_n < \infty$ a.s.;
    \item[(A2)] $\sum_{n=0}^{\infty} x_n < \infty$ a.s.
\end{enumerate}
Then 
\begin{enumerate}
  \item $\lim_{n \to \infty} z_n$ exists and is finite a.s.;
  \item $\sum_{n=0}^{\infty} y_n < \infty$ a.s.
\end{enumerate}
\end{lemma}
Here, the zero-order term in Section~\ref{sec:intro} refers to $\qty{x_n}$ and the summable condition refers to (A2).
Lemma~\ref{thm:rs} is an asymptotic result, which is recently extended to a non-asymptotic version by \citet{liu2022almost,karandikar2024convergence},
under some additional assumptions.
\begin{lemma}[Theorem 2 of \citealt{karandikar2024convergence}]
\label{aux:thm2}
Consider $\qty{z_n}$ in~\eqref{eq:rs}.
Let the non-negative real-valued random sequence $\qty{b_n}_{n \geq 0}$ be such that $y_n = b_n z_n$.
For some $\eta \in (0,1)$, assume that
\begin{enumerate}
    \item[(A1)] $\sum_{n=0}^{\infty} a_n< \infty$, $\sum_{n=0}^{\infty} x_n < \infty$, $\sum_{n=0}^{\infty} b_n = \infty$ a.s.;
    \item[(A2)] There exists some $n_1 \geq 1$ such that for all $n \geq n_1$, $b_n \geq \frac{\eta}{n}$ a.s.;
    \item[(A3)] $\sum_{n=1}^{\infty} (n+1)^{\eta}x_n < \infty$, $\sum_{n=1}^{\infty} (b_n - \frac{\eta}{n}) = \infty$ a.s.
\end{enumerate}
Then $\lim_{n \to \infty} n^{\eta}z_n = 0$ a.s.
\end{lemma}
Here (A3) is still a summable condition on the zero-order term.

Sometimes practitioners \citep{bertsekas1996neuro,zhang2020global,barakat2021convergence,sebbouh2021almost,xin2021improved,qian2026almost} consider a special identification
of $\qty{a_n, x_n, y_n}$ in~\eqref{eq:rs}. 
Namely,
they consider a deterministic sequence $\qty{T_n \in (0, 1)}_{n\geq 0}$, some positive scalars $\alpha, \xi$, and use
$a_n=0$, $x_n = \xi T_n^2$, and $y_n = \alpha T_n z_n$,
yielding
\begin{align}
  \label{eq: prior special rs}
  \textstyle\E_n\qty[z_{n+1}] \leq (1 - \alpha T_n) z_n + \xi T_n^2.
\end{align}
The $\qty{T_n}$ usually results from the learning rates of stochastic approximation and reinforcement learning algorithms so it satisfies the Robbins-Monro condition,
i.e.,
$\sum_{n=0}^\infty T_n = \infty, \sum_{n=0}^\infty T_n^2 < \infty$.
Then naturally,
the zero-order term $\xi T_n^2$ is summable.
However,
many important RL algorithms (see Section~\ref{sec:linear q} for examples) do not reduce to the nice form above.
For those algorithms, the corresponding zero-order term is $T_n$ instead of $T_n^2$,
i.e.,
\begin{align}
  \label{eq:special rs}\tag{RS-Special}
    \textstyle\E_n\qty[z_{n+1}] \leq (1 - \alpha T_n) z_n + \xi T_n.
\end{align}
Unfortunately,
the zero-order term is no longer summable and is only square-summable under the Robbins-Monro condition.
Since $T_n \geq T_n^2$,
this situation essentially allows more noise,
making $\qty{z_n}$ further away from being a supermartingale.
To our knowledge, 
no existing extension of the Robbins-Siegmund theorem is able to handle this situation.
This is the gap that this paper aims to bridge.

\tb{Notations.}
For a real number $z$, we denote its positive part by $z^{+} \doteq \max(z, 0)$. Note that the operator $(\cdot)^+$ satisfies 
\begin{align}
    \textstyle\abs{z_1^+ - z_2^+} \leq \abs{z_1 - z_2}\quad \forall z_1, z_2 \in \R.
    \label{eq:+ diff}
\end{align}
This inequality can be easily proved via the identity $z^+ = \frac{1}{2} z + \frac{1}{2} \abs{z}$ and the triangle inequality.
For an interval $Z \subset \mathbb{R}$,
we denote $d(z,Z) \doteq \inf_{\tilde{z} \in Z} \abs{z - \tilde{z}}$ as the distance from $z$ to $Z$. 

\section{Asymptotic Extensions of Robbins-Siegmund Theorem}
\label{sec:special}

We now proceed to analyze~\eqref{eq:special rs}
under the Robbins-Monro condition $\sum_{n=0}^\infty T_n = \infty$ and $\sum_{n=0}^\infty T_n^2 < \infty$.
Consider a special case where $z_n$ is deterministic.
Then~\eqref{eq:special rs} degenerates to 
\begin{align}
  z_{n+1} \leq (1 - \alpha T_n) z_n + \xi T_n.
\end{align}
Under the Robbins-Monro condition,
it is trivial to show that $\lim_{n\to\infty} d(z_n, [0, \xi / \alpha]) = 0$,
i.e.,
$\qty{z_n}$ converge to a bounded interval.
Building on this special case,
one might then optimistically conjecture that under the Robbins-Monro condition,
the $\qty{z_n}$ in~\eqref{eq:special rs} also converge to a bounded interval almost surely.
Unfortunately,
this is impossible.
Below we provide an example where the $\qty{z_n}$ in~\eqref{eq:special rs} can be unbounded almost surely even when the Robbins-Monro condition holds.
\begin{example}[Divergence with Robbins-Monro Condition]
  \label{eg:main}
Define $T_n \doteq (n+1)^{-3/4}$.
Let $\qty{X_{n+1}}_{n\ge 0}$ be independent with
\begin{align}
    X_{n+1}\doteq
   \begin{cases}
      A_n &\text{with probability } p_n,\\
      0   &\text{with probability } 1-p_n,
   \end{cases}
   \qquad n\ge 0,
\end{align}
where $A_n \doteq (n+1)^{1/4}, p_n \doteq \frac{T_n}{A_n} = \frac{1}{n+1}$. 
Starting from $z_0=0$ and with filtration $\fF_n=\sigma(X_1,\dots,X_n)$,
define recursively
\begin{equation}
\label{eq:recurrence}
   \textstyle z_{n+1} = (1-T_n)\,z_n + X_{n+1}, \qquad n\geq 0.
\end{equation}
Then 
\begin{enumerate}
  \item $\qty{z_n}$ satisfies $\E_n\qty[z_{n+1}]=(1-T_n)z_n+T_n$, i.e., \eqref{eq:special rs};
\item $\limsup_{n\to\infty} z_n = \infty$ a.s.
\end{enumerate}
\end{example}
\begin{proof}
$\qty{T_n}$ apparently fulfills the Robbins-Monro condition.
Since $X_{n+1}$ is independent of $\fF_n$ and $\E_n[X_{n+1}]=p_nA_n=T_n$, we have
$\E_n\qty[z_{n+1}]=(1-T_n)z_n+T_n$.
That is, \eqref{eq:special rs} holds with $\alpha=\xi=1$.

Define the events
$\Xi_n \doteq \qty{X_{n+1}=A_n}$.
They are independent with
$\Pr(\Xi_n)=p_n=1/(n+1)$ and $\sum_n \Pr(\Xi_n)=\infty$. By the second
Borel-Cantelli lemma, $\Pr(\qty{\Xi_n \, \text{infinitely often}})=1$. On $\Xi_n$,
\begin{align}
  \textstyle z_{n+1} = (1-T_n)z_n + A_n \geq A_n = (n+1)^{1/4}.
\end{align}
Since $\Xi_n$ occurs infinitely often almost surely,
$z_{n+1} > (n + 1)^{1/4}$ occurs infinitely often almost surely,
forcing $\limsup_{n\to\infty} z_n=\infty$ almost surely.
\end{proof}

Example~\ref{eg:main} shows that the Robbins-Monro condition alone does not prevent
the potential divergence of $\qty{z_n}$ in \eqref{eq:special rs}. 
Intuitively, the divergence in Example~\ref{eg:main} results from the spikes of ${X_{n+1}}$.
$X_{n+1}$ is mostly 0 but occasionally spikes to an increasingly large value $A_n$. 
Between the spikes, the factor $1 - T_n$ contracts $z_n$.
But the spikes occur infinitely often,
making $z_n$ grow unbounded almost surely.
This motivates us to introduce additional assumptions to control such spikes,
or equivalently,
control the growth of $z_{n+1}$ from $z_n$.
\begin{theorem}
\label{thm:bound}
Consider $\qty{z_n}$ in~\eqref{eq:special rs} with   
$\qty{T_n\in (0,1)}_{n \geq 0}$ being a sequence of deterministic real numbers.
Suppose 
\begin{enumerate}
    \item[(A1)] $\sum_n T_n = \infty$ and $\sum_n T_n^2 < \infty$;
    \item[(A2)] There exists some constant $B_{\tref{thm:bound},1}$, such that for all $n\geq 0$,
    \begin{align}
        \abs{z_{n+1} - z_n} \leq B_{\tref{thm:bound},1}T_n (z_n + 1) \qq{a.s.}
    \end{align}
\end{enumerate}
Then there exists some constant $B_\tref{thm:bound}$, such that
\begin{align}
    \textstyle\lim_{n \to \infty} d\qty(z_n, [0, B_\tref{thm:bound}])=0 \qq{a.s.}
\end{align}
\end{theorem}
\begin{proof}
We will identify $B_\tref{thm:bound} = \frac{\xi}{\alpha}$.
Let $u_n \doteq d(z_n, [0, \frac{\xi}{\alpha}]) = \qty(z_n -\frac{\xi }{\alpha})^+$. Now we consider two cases.

\tb{(1) When $u_n >0$.} Then $u_n = z_n -\frac{\xi }{\alpha}$,
 we have 
 \begin{align}
     \label{eq:decomp_u}
     \textstyle u_{n+1}^2=\qty(\qty(z_{n+1} - \frac{\xi}{\alpha})^+)^2\leq \qty(z_{n+1} - \frac{\xi}{\alpha})^2 =\qty((z_{n+1}-z_n)+z_n-\frac{\xi}{\alpha})^2 = \qty((z_{n+1}-z_n)+u_n)^2.
 \end{align}
Thus
\begin{align}
  \textstyle\E_n[u_{n+1}^2]&\textstyle \leq\E_n\qty[\qty(u_n+(z_{n+1}-z_n))^2] \explain{By \eqref{eq:decomp_u}}\\
   &\textstyle = u_n^2 + 2u_n\,\E_n\qty[z_{n+1}-z_n] + \E_n\qty[(z_{n+1}-z_n)^2] \\
   &\textstyle\leq u_n^2 + 2 u_n \qty(-\alpha T_n z_n + \xi T_n)
        + B_{\tref{thm:bound},1}^2 T_n^2
          \qty(z_n + 1)^2 \explain{By \eqref{eq:special rs} and  Theorem~\ref{thm:bound} (A2)}\\
    &\textstyle= u_n^2 + 2 u_n \cdot \qty(-\alpha T_n u_n)
        + B_{\tref{thm:bound},1}^2 T_n^2
          \qty(u_n + \frac{\xi }{\alpha } + 1)^2 \explain{Since $u_n = z_n -\frac{\xi }{\alpha}$}\\
   &\textstyle\leq u_n^2 - 2\alpha T_n u_n^2
        + 2B_{\tref{thm:bound},1}^2 T_n^2u_n^2 + 2B_{\tref{thm:bound},1}^2 T_n^2\qty(\frac{\xi}{\alpha} + 1)^2 \\
   &\textstyle= \qty(1 - 2\alpha T_n + 2B_{\tref{thm:bound},1}^2 T_n^2) u_n^2
        + 2B_{\tref{thm:bound},1}^2
          \qty(1 + \frac{\xi }{\alpha })^2 T_n^2.
\end{align}

\tb{(2) When $u_n=0$.} Then $z_n\leq \frac{\xi }{\alpha }$. By \eqref{eq:+ diff} we obtain 
\begin{align}
\label{eq:diff_u_s}
    \textstyle\abs{u_{n+1}} = \abs{u_{n+1}- u_n} \leq \abs{(z_{n+1}-\frac{\xi}{\alpha})-(z_n-\frac{\xi}{\alpha})} \leq \abs{z_{n+1}-z_n}.
\end{align}
Thus 
\begin{align}
  \textstyle\E_n\qty[u_{n+1}^2]
  \leq&\textstyle \E_n\qty[\abs{z_{n+1}-z_n}^2]\explain{By \eqref{eq:diff_u_s}}\\
  \leq&\textstyle B_{\tref{thm:bound},1}^2
      \qty(1 + z_n)^2 T_n^2 \explain{By Theorem~\ref{thm:bound} (A2)}\\
      \leq&\textstyle B_{\tref{thm:bound},1}^2
      \qty(1 + \frac{\xi }{\alpha })^2 T_n^2 \explain{Since $z_n\in\qty[0, \frac{\xi }{\alpha }]$}\\
  \leq& \textstyle \qty(1 - 2\alpha T_n + 2B_{\tref{thm:bound},1}^2 T_n^2) u_n^2
        + 2B_{\tref{thm:bound},1}^2
          \qty(1 + \frac{\xi }{\alpha })^2 T_n^2. \explain{Since $u_n =0$}
\end{align}
Combining two cases we have
\begin{align}
\label{eq:almost_mrtg}
  \textstyle\E_n\qty[u_{n+1}^2]
  \leq \qty(1 - 2\alpha T_n + 2B_{\tref{thm:bound},1}^2 T_n^2) u_n^2
     + 2B_{\tref{thm:bound},1}^2
       \qty(1 + \frac{\xi }{\alpha })^2 T_n^2.
\end{align}
We now proceed to apply Lemma~\ref{thm:rs}. We identify in \eqref{eq:rs}
\begin{equation}
    \textstyle z_n \leftrightarrow u_n^2, a_n \leftrightarrow 2B_{\tref{thm:bound},1}^2T_n^2,
     x_n \leftrightarrow 2B_{\tref{thm:bound},1}^2\qty(1 + \frac{\xi }{\alpha })^2T_n^2, 
     y_n \leftrightarrow 2\alpha T_n u_n^2.
\end{equation}
As $\sum_n T_n^2 < \infty$, we have $\sum_n a_n < \infty$ and $\sum_n x_n < \infty$ a.s. Applying Lemma~\ref{thm:rs} then gives: 1. $\lim_{n\to \infty} u_n^2$ exists and is finite almost surely; 2. $\sum_{n=0}^\infty y_n < \infty$ a.s. That is, $\sum_{n=0}^\infty T_n u_n^2 < \infty$ a.s.

We now illustrate that the only possibility is $\lim_{n \to \infty} u_n^2 = 0$ a.s. For a fixed sample path $\omega$, consider $u_n(\omega)^2$ as a deterministic sequence (abbreviated as $u_n^2$).
If the limit $\zeta_0 \doteq \lim_{n\to\infty} u_n^2$ exists and $\zeta_0 \in (0, \infty)$, then there exists $N$ such that $u_n^2 \geq \zeta_0/2$ for all $n \geq N$. Consequently,
\begin{align}
    \textstyle\sum_{n=N}^{\infty} T_n u_n^2 \geq \frac{\zeta_0}{2}\sum_{n=N}^{\infty} T_n = \infty,
\end{align}
where the last equality holds because $\sum_{n=0}^\infty T_n = \infty$. This contradicts the convergence of $\sum_{n=0}^\infty T_n u_n^2$.
By definition, we have $\zeta_0 \geq 0$. Therefore, the only possibility is $\zeta_0 = 0$. That is, $0= \lim_{n\to\infty} u_n^2 = \lim_{n\to\infty} \qty(\qty(z_n - \frac{\xi }{\alpha })^+)^2 = \lim_{n\to\infty} d\qty(z_n, \qty[0, \frac{\xi }{\alpha}])^2 $. Denoting $B_\tref{thm:bound}\doteq \frac{\xi }{\alpha }$ then completes the proof.
\end{proof}
As demonstrated by Example~\ref{eg:main}, the Robbins-Monro condition alone cannot prevent divergence when the zero-order term is only square-summable. To address this, Assumption (A2) of Theorem~\ref{thm:bound} explicitly controls the incremental growth by prescribing that $z_{n+1}$ can grow at most linearly in $z_n$, gated by $T_n$. This additional structure successfully rules out the pathological spikes seen in Example~\ref{eg:main}. While exploring weaker growth conditions (e.g., sub-linear growth) to establish the minimal sufficient conditions remains an interesting open question, the affine form in (A2) is highly natural for our setting and is well suited to our proof technique. In particular, we will demonstrate in Section~\ref{sec:linear q} that this additional assumption indeed trivially holds for important reinforcement learning applications, such as $Q$-learning with linear function approximation (cf. Lemma~\ref{lem:bound diff}).

We now apply the same idea to study~\eqref{eq:rs}.
Similarly,
we want to relax the summable assumption (A2) in Lemma~\ref{thm:rs}.
To this end,
we additionally introduce a square-summable control on the growth between consecutive steps and an overall drift that is non-positive once the process is overly large.



\begin{theorem}
\label{thm:rs general}
Consider $\qty{z_n}$ in~\eqref{eq:rs}. Suppose that
\begin{enumerate}
    \item[(A1)]
        $\sum_n a_n < \infty$ a.s.;
    \item[(A2)] $\abs{z_{n+1} - z_n} \leq b_n(z_n + 1)$ a.s.,
        where $\qty{b_n}_{n\geq 0}$ is a non-negative real-valued random sequence adapted to $\qty{\fF_n}_{n\geq 0}$ and satisfies $\sum_n b_n^2 < \infty$ a.s.;
    \item[(A3)] There exists some constant $B_\tref{thm:rs general}$ such that 
    \begin{align}
      z_n > B_\tref{thm:rs general} \implies x_n - y_n  \leq  -c_n (z_n - B_\tref{thm:rs general}),
    \end{align}
    where $\qty{c_n }_{n\geq 0}$ is a non-negative real-valued random sequence adapted to $\qty{\fF_n}_{n\geq 0}$ and satisfies $\sum_n c_n = \infty$ a.s.
    
\end{enumerate}
Then
\begin{align}
    \textstyle\lim_{n \to \infty} d\qty(z_n, [0, B_\tref{thm:rs general}])=0 \qq{a.s.}
\end{align}
\end{theorem}
The proof follows a similar idea as in Theorem~\ref{thm:bound} and is provided in Appendix~\ref{sec:proof rs general}.
We now elaborate on the connection between Theorem~\ref{thm:rs general} and Theorem~\ref{thm:bound}.
When reducing \eqref{eq:rs} to \eqref{eq:special rs}, we set $a_n = 0$, $x_n = \xi T_n$, and $y_n = \alpha T_n z_n$, so that $x_n - y_n = -\alpha T_n(z_n - \xi/\alpha)$.
This gives $B_\tref{thm:rs general} = \xi/\alpha$ and $c_n = \alpha T_n$ in Theorem~\ref{thm:rs general}.
With this reduction, Theorem~\ref{thm:rs general} (A1) is absent in Theorem~\ref{thm:bound} since $a_n = 0$.
Instead, Theorem~\ref{thm:bound} (A1) directly imposes the Robbins-Monro condition. 
The assumptions correspond as follows:
The requirement $\sum_n T_n = \infty$ ensures that $1 - T_n$ generates enough contractive drift over time.
Specifically, when $z_n > B_\tref{thm:rs general} = \xi/\alpha$, we have $x_n - y_n = -\alpha T_n(z_n - \xi/\alpha) \leq -c_n(z_n - B_\tref{thm:rs general})$ with $c_n = \alpha T_n$, and $\sum_n T_n = \infty$ ensures $\sum_n c_n = \infty$, satisfying Theorem~\ref{thm:rs general} (A3).
The requirement $\sum_n T_n^2 < \infty$ ensures that the growth rate is square-summable, corresponding to $\sum_n b_n^2 < \infty$ in Theorem~\ref{thm:rs general} (A2).


Notably,
we are not claiming that Theorems~\ref{thm:bound} \&~\ref{thm:rs general} are stronger than Lemma~\ref{thm:rs}.
Indeed,
we remove the summable assumption at the cost of introducing new assumptions and the convergence is to a bounded set instead of a point.
Our key contribution and novelty lies in that we are the first to obtain convergence results for almost supermartingales without the summable assumption on the zero-order term.
And according to Example~\ref{eg:main},
our results might be the best that one can hope for.
Additionally,
we do demonstrate in Section~\ref{sec:linear q} that our results have significant applications in reinforcement learning. 


\section{Nonasymptotic Extensions of Robbins-Siegmund Theorem}
\label{sec:as con}
Having established asymptotic extensions of the Robbins-Siegmund theorem,
we in this section further establish nonasymptotic convergence rates.
We start with~\eqref{eq:special rs} and an almost sure convergence rate.


\begin{theorem}
    \label{thm:rate}
    Consider $\qty{z_n}_{n\ge 0}$ in~\eqref{eq:special rs} and 
    let all conditions in Theorem~\ref{thm:bound} hold. 
    Additionally, assume there exists some constant $\eta > 0$ such that
    \begin{enumerate}
        \item[(A1)] $\liminf_{n \to \infty} nT_n > \frac{\eta}{2\alpha }$; 
        \item[(A2)] $\sum_n (n+1)^\eta T_n^2 < \infty$.
    \end{enumerate} 
    Then
    \begin{equation}
    \textstyle\lim_{n \to \infty} n^{\eta/2}d\qty(z_n, [0, B_\tref{thm:bound}]) = 0 \qq{a.s.}
    \end{equation}
\end{theorem}
Similar to the proof of Theorem~\ref{thm:bound} where we construct auxiliary sequences to invoke Lemma~\ref{thm:rs},
we prove Theorem~\ref{thm:rate} by constructing auxiliary sequences to invoke Lemma~\ref{aux:thm2},
which is provided in Appendix~\ref{proof:sec rate}.

If $T_n$ decays in an inverse linear manner, 
we can further obtain high probability concentration bound and $L^p$ convergence rates.
\begin{assumption} 
\label{asp:T_n} 
$T_n = \frac{C_\tref{asp:T_n}}{n+n_0 }$
for some positive constants $n_0$ and $C_\tref{asp:T_n} > 1/\alpha$.
\end{assumption}

\begin{theorem}
\label{thm:concentration rs}
Consider $\qty{z_n}$ in~\eqref{eq:special rs}, $z_0$ is deterministic, and let all conditions in Theorem~\ref{thm:bound} hold.
Let Assumption~\ref{asp:T_n} hold.
Then there exist 
some integers $B_\tref{thm:concentration rs}$ and $B'_\tref{thm:concentration rs}$ such that 
for any $\delta > 0$, it holds, with probability at least $1 - \delta$, that for all $n \geq 0$,
\begin{equation}
    \textstyle\qty(d\qty(z_n, [0, B_\tref{thm:bound}]))^2 \leq B'_\tref{thm:concentration rs}\frac{1}{n+n_0}\left[\ln\left(\frac{B_\tref{thm:concentration rs}}{\delta}\right) + 1 + \ln(n+n_0)\right]^{B_\tref{thm:concentration rs}}.
\end{equation}
\end{theorem} 
We prove Theorem~\ref{thm:concentration rs} by an inductive approach inspired
by \citet{chen2025concentration},
which is provided in Appendix~\ref{sec:proof concentration rs}.
We argue that Theorem~\ref{thm:concentration rs} is a strong result as it is a maximal bound that simultaneously holds for all $n$ and it achieves an exponential tail.

\begin{corollary}
  \label{cor:lp}
Consider $\qty{z_n}$ in~\eqref{eq:special rs}, $z_0$ is deterministic, and let all conditions in Theorem~\ref{thm:bound} hold.
Let Assumption~\ref{asp:T_n} hold.
Then for any integer $p \geq 2$ and any $n \geq 0$,
  \begin{align}
      \mathbb{E}\qty[\qty(d\qty(z_n, [0, B_\tref{thm:bound}]))^{2p}] \leq \qty(a(n) b(n)^{B_\tref{thm:concentration rs}})^p +  B_\tref{thm:concentration rs}p a(n)^p \exp(b(n)) ((B_\tref{thm:concentration rs}p)!),
  \end{align}
  where $a(n) \doteq B'_\tref{thm:concentration rs}/(n+n_0)$, $b(n) \doteq \ln B_\tref{thm:concentration rs} + 1 + \ln\qty(n+n_0)$.
\end{corollary}
Corollary~\ref{cor:lp} can be obtained easily by integrating the sub-Weibull tails from Theorem~\ref{thm:concentration rs}.
This integration procedure is identical to the proof of Corollary 1 in \citet{qian2026almost} and is thus omitted.


Together, Theorems~\ref{thm:rate} \& \ref{thm:concentration rs} and Corollary~\ref{cor:lp} provide a comprehensive set of nonasymptotic tools (including almost sure rates, high-probability concentration bounds, and $L^p$ rates) for analyzing stochastic iterates that fit the \eqref{eq:special rs} template. The following sections will demonstrate the power of this framework by applying it to stochastic approximation and linear $Q$-learning.
We close this section by providing the counterpart of Theorem~\ref{thm:rate} for~\eqref{eq:rs}.
Notably, we do not provide the counterparts of Theorem~\ref{thm:concentration rs} and Corollary~\ref{cor:lp} for~\eqref{eq:rs} because it is not clear how to embed the structure in Assumption~\ref{asp:T_n} into~\eqref{eq:rs}.

\begin{theorem}
    \label{thm:rate general}
    Consider $\qty{z_n}$ in~\eqref{eq:rs} and
    let all the conditions of Theorem~\ref{thm:rs general} hold. Additionally, assume there exists some constant $\eta > 0$ such that 
    \begin{enumerate}
        \item[(A1)] $\liminf_{n \to \infty} nc_n > \frac{\eta}{2}$;
        \item[(A2)] $\sum_n (n+1)^\eta (a_n + b_n^2) < \infty$ a.s.
    \end{enumerate}
    Then
    \begin{equation}
    \lim_{n \to \infty} n^{\eta/2}d\qty(z_n, [0, B_\tref{thm:rs general}]) = 0 \qq{a.s.}
    \end{equation}
\end{theorem}
The proof follows a procedure similar to that of Theorem~\ref{thm:rate} and is provided in Appendix~\ref{sec:proof rate general}.

\section{Stochastic Approximation}
\label{sec:sa}
We now use Theorems~\ref{thm:rate} \& \ref{thm:concentration rs} and Corollary~\ref{cor:lp} to derive new convergence results for stochastic approximations.
Specifically, 
we consider stochastic approximation algorithms in the form of  
\begin{align}
  \label{eq:sa update}
w_{t+1} = w_t + \alpha_t H(w_t, Y_{t+1}).
\tag{SA}
\end{align}
Here $\qty{w_t \in \mathbb{R}^d}$ is the stochastic iterates, $w_0$ is deterministic, $\qty{\alpha_t}$ is a sequence of learning rates, $\qty{Y_t}$ is a
time-inhomogeneous Markov chain evolving in a finite space $\mathcal{Y}$, and $H : \mathbb{R}^d \times \mathcal{Y} \to \mathbb{R}^d$ is the function that 
generates the incremental update. 
Importantly, 
we consider the setting where the transition kernel of $\qty{Y_t}$ is controlled by $\qty{w_t}$ itself.
This setting admits important applications such as linear SARSA \citep{rummery1994line,zou2019finite,zhang2023convergence} and linear $Q$-learning \citep{baird1995residual,meyn2024bellmen,liu2025linearq}.
In such settings,
the convergence of $\qty{w_t}$ to a point is usually not attainable (see, e.g., \citealt{zhang2023convergence,meyn2024bellmen,liu2025linearq}) unless strong assumptions are made (see, e.g., \citealt{zou2019finite}).
We, therefore,
aim to establish the convergence to a bounded set instead.
The high-level idea of our approach is to fit~\eqref{eq:sa update} into the template~\eqref{eq:special rs}.
To this end,
we use the skeleton iterates technique from~\citet{qian2026almost} to work with the Markovian nature of $\qty{Y_t}$ and use a novel error decomposition from~\citet{liu2025linearq} to work with the time-inhomogeneous nature of $\qty{Y_t}$. 
One of our main contributions lies in a new analysis that successfully controls the complex bias terms arising from this integration, making these two powerful techniques compatible for the first time. 
After obtaining~\eqref{eq:special rs},
our new results can then take over.

This contribution is crucial because \citet{qian2026almost} have to assume that $\qty{Y_t}$ is time-homogeneous so they can use the original Robbins-Siegmund theorem, while 
\citet{liu2025linearq} have to take total expectation on both sides of~\eqref{eq:special rs} to derive the rate at which the second moments of $\qty{w_t}$ converge to a bounded interval.
\citet{liu2025linearq} thus miss the opportunity to derive an almost sure convergence rate, a high probability concentration bound, and an $L^p$ convergence rate. 
We now list our assumptions,
many of which are adapted from \citet{qian2026almost,borkar2025ode,liu2025linearq}.


\begin{assumptionp}{LR1}[Assumption LR1 in \citealt{qian2026almost}]
  \label{asp:lr1}
  The learning rate has the form 
  \begin{align}
    \textstyle\alpha_t = \frac{C_\alpha}{(t+3)^{\nu}} \qq{with} \nu \in \left(\frac{2}{3}, 1\right].
  \end{align}
\end{assumptionp}
\begin{assumptionp}{LR2}[Assumption LR2 in \citealt{qian2026almost}]
  \label{asp:lr2}
  The learning rate has the form 
  \begin{align}
    \textstyle\alpha_t = \frac{C_\alpha}{(t+3)\ln^\nu(t+3)} \qq{with} \nu \in (0, 1).
  \end{align}
\end{assumptionp}
The two forms of learning rates are from the skeleton iterates technique of \citet{qian2026almost}.
Assumption~\ref{asp:lr1} is a standard choice of learning rate except that $\nu$ cannot be smaller than $2/3$,
which is an inherent limitation of the skeleton iterates technique.
This learning rate will be used for obtaining an almost sure convergence rate of $\qty{w_t}$ by invoking Theorem~\ref{thm:rate}.
Assumption~\ref{asp:lr2} is specially designed such that the resulting $T_n$ in~\eqref{eq:special rs} can verify Assumption~\ref{asp:T_n}.
This learning rate will be used for obtaining high probability concentration bound and $L^p$ convergence rate by invoking Theorem~\ref{thm:concentration rs} and Corollary~\ref{cor:lp}.
We refer the reader to \citet{qian2026almost} for more discussion about the two learning rates.

\begin{assumption}[Assumption~3.1 \& 3.2 in \citealt{zhang2022global}]
  \label{asp:markov chain}
  There exists a family of parameterized
  transition matrices $\Lambda_P \doteq \qty{P_w \in \mathbb{R}^{\abs{\mathcal{Y}} \times \abs{\mathcal{Y}}}\mid w \in \mathbb{R}^d}$ such
  that $\operatorname{Pr}(Y_{t+1} \mid Y_t) = P_{w_t}(Y_t, Y_{t+1})$. Furthermore, let $\bar{\Lambda}_P$
  denote the closure of $\Lambda_P$. Then for any $P \in \bar{\Lambda}_P$, the time-homogeneous Markov chain induced by $P$ is irreducible
  and aperiodic. 
\end{assumption}
Assumption~\ref{asp:markov chain} looks prohibitive but as can be seen soon in Section~\ref{sec:linear q},
this assumption can be trivially verified in applications.
Under Assumption~\ref{asp:markov chain}, for any $w$, the Markov chain induced by $P_w$ has a unique stationary distribution, 
which we denote as $d_{\mathcal{Y}, w}$. This allows us to define the expected update as
$h(w) \doteq \E_{Y \sim d_{\fY, w}}\qty[H(w, Y)]$.
Additionally, by Lemma~1 of \citet{zhang2022global}, Assumption~\ref{asp:markov chain} implies 
that the Markov chains in $\Lambda_P$ mix both geometrically and 
uniformly. 
In other words, 
there exist constants $C_\tref{asp:markov chain} > 0$ and 
$\tau \in [0, 1)$, independent of $w$, such that
\begin{align}
\label{eq:mixing}
\textstyle\sup_{w,y} \sum_{y'} \abs{P_w^n(y, y') - d_{\mathcal{Y},w}(y')} \leq C_\tref{asp:markov chain} \tau^n. 
\end{align}
We then define 
\begin{align}
\label{eq:tau}
\tau_\alpha \doteq \min\qty{n \geq 0 \mid C_\tref{asp:markov chain}\tau^n \leq \alpha}
\end{align}
to denote the number of steps that the Markov chain needs to mix to an accuracy $\alpha$. 
We now introduce the norm we will work with in the rest of the paper.
Let $\langle \cdot, \cdot \rangle$ be an arbitrary inner product for vectors in $\R[d]$. 
Let $\norm{\cdot}$ be the norm induced by this inner product (i.e., the inner product norm for vectors and the corresponding induced norm for matrices). 
For a non-negative real number $r$, we denote the closed ball in $\mathbb{R}^d$ centered at $0$ with radius $r$ by $\fBB(r) \doteq \qty{x \in \mathbb{R}^d \mid \norm{x} \leq r}$.

\begin{assumption}[Condition (30) in \citealt{borkar2025ode}]
    \label{asp:P_lip'}
    There exists a constant $C_\tref{asp:P_lip'}$ such that for $\forall w_1, w_2$
    \begin{align}
        \textstyle \norm{P_{w_1} - P_{w_2}} \leq \frac{C_\tref{asp:P_lip'}}{1+\norm{w_1} + \norm{w_2}}\norm{w_1 - w_2}.
    \end{align}
\end{assumption}

\begin{assumption}[Assumption (A2) in \citealt{borkar2025ode}]
\label{asp:lip}
    There exists some constant $L_h$, such that for $\forall w_1, w_2, y$
    \begin{align}
        &\norm{H(w_1,y) - H(w_2,y)} \leq L_h\norm{w_1 - w_2} \qq{and} \norm{H(0,y)} \leq L_h.
    \end{align}
\end{assumption}
Assumption~\ref{asp:P_lip'} is stronger than the standard Lipschitz continuity assumption (cf. Assumption~\ref{asp:lip}) and is the key to handling the time-inhomogeneous nature of $\qty{Y_t}$.
The earliest work we are aware of using this stronger form of Lipschitz continuity to ensure the Lipschitz continuity of the following composition function is \citet{konda2002thesis} (Assumption 5.1.3). 
More recent usages include \citet{lauand2024revisiting,borkar2025ode} and the follow-up work \citet{liu2025linearq}.
\begin{assumption}
  \label{asp:bound inner}
    There exist some deterministic constants $C_{\tref{asp:bound inner},1} > 0$ and $C_{\tref{asp:bound inner},2}$ such that 
    \begin{align}
    \indot{w}{h(w)} \leq -C_{\tref{asp:bound inner},1} \norm{w}^2 + C_{\tref{asp:bound inner},2}.
    \end{align}
\end{assumption}
Assumption~\ref{asp:bound inner} ensures that the expected update direction $h(w)$ will decay $\norm{w}$ at least when $w$ is overly large.
The factor $C_{\tref{asp:bound inner},2}$ is usually 0
when considering stochastic approximation algorithms with GAS ODEs.
But for a stochastic approximation algorithm with a non-GAS ODE, see, e.g., \citet{meyn2024bellmen},
it is usually strictly positive,
constituting a main obstacle for analyzing the convergence of such algorithms.

\begin{theorem}[Almost Sure Convergence Rate for SA]
    \label{thm:rate sa}
    Let Assumptions~\ref{asp:markov chain} - \ref{asp:bound inner} and \ref{asp:lr1} hold.
    Let $\qty{w_t}$ be the iterates generated by \eqref{eq:sa update}. Then there exists some $B_\tref{thm:rate sa}$, 
    such that the following statements hold.
    If $\nu < 1$, then for any $\zeta \in \qty(0, \frac{3}{4}\nu - \frac{1}{2})$,
\begin{align}
\textstyle\lim_{t\rightarrow\infty} t^{\zeta}d\qty(w_t, \fBB\qty(B_\tref{thm:rate sa})) = 0 \qq{a.s.}
\end{align}
If $\nu = 1$, then for any $\zeta \in (0,\frac{1}{2})$ and $\nu_1 > 0$,
\begin{align}
\textstyle\lim_{t\rightarrow\infty} \exp\left(\zeta \ln^{1/(1+\nu_1)} t\right) d\qty(w_t, \fBB\qty( B_\tref{thm:rate sa})) = 0 \qq{a.s.}
\end{align}
\end{theorem}

\begin{theorem}[Concentration for SA]
\label{thm:concentration sa}
Let Assumptions~\ref{asp:markov chain} - \ref{asp:bound inner} and \ref{asp:lr2} hold. Then for any $\nu \in
(0,1)$, there exist some deterministic constants $\bar{C}_{\alpha}, B_\tref{thm:concentration sa}, B'_\tref{thm:concentration sa}$, and integer $C_\tref{lem:concentration sa0}$ such that
the iterates $\{w_t\}$ generated by \eqref{eq:sa update} satisfy the following concentration property whenever
$C_{\alpha} \geq \bar{C}_{\alpha}$: for any $\delta \in (0,1)$, it holds, with probability at least $1-\delta$, that for all $t \geq 0$,
\begin{align}
\textstyle d\qty(w_t, \fBB\qty(B_\tref{thm:rate sa})) \leq B_\tref{thm:concentration sa}\exp\left(\frac{-\ln^{1-\nu}(t + 1)}{2(1 - \nu)}\right)\left[\ln\left(\frac{1}{\delta}\right) + B'_\tref{thm:concentration sa} + \frac{\ln^{1-\nu}(t + 1)}{1 - \nu}\right]^{C_\tref{lem:concentration sa0}/2}.
\end{align}
\end{theorem} 
\begin{corollary}[$L^p$ Convergence Rates for SA]
  \label{cor:lp sa}
Let conditions of Theorem~\ref{thm:concentration sa} hold. Then the iterates $\qty{w_t}$ generated by \eqref{eq:sa update} satisfy for any integer $p > 2$ and any $t \geq 0$
  \begin{align}
      \E\qty[\qty(d\qty(w_t, \fBB\qty(B_\tref{thm:rate sa})))^{p}] \leq \qty(a(t) b(t)^{\lceil C_\tref{lem:concentration sa0}/2 \rceil})^p + \lceil C_\tref{lem:concentration sa0}/2 \rceil p a(t)^p \exp(b(t)) ((\lceil C_\tref{lem:concentration sa0}/2 \rceil p)!),
  \end{align}
  where $a(t) \doteq B_\tref{thm:concentration sa}\exp\left(\frac{-\ln^{1-\nu}(t + 1)}{2(1 - \nu)}\right)$, $b(t) \doteq B'_\tref{thm:concentration sa} + \frac{\ln^{1-\nu}(t + 1)}{1 - \nu}$.
\end{corollary}
We shall shortly present a sketch to highlight the key idea in proving Theorems~\ref{thm:rate sa} \&~\ref{thm:concentration sa} and Corollary~\ref{cor:lp sa} with the complete proof provided in Appendix~\ref{proof:sec sa}.
Table~\ref{table:sa} highlights our improvements over prior works.
\begin{table*}[h!]
    \centering
  \begin{tabular}{c|ccc|c|c}
  \hline
  & \multicolumn{3}{c|}{Convergence Modes} & Metric & $P_{w_t}$ \\\hline
  & a.s. &  $\Pr(\dots) \geq 1 - \delta$ & $L^p$ & 
  \\ \hline
  \citet{qian2026almost} & \cmark & \cmark & \cmark & Distance to a point  &  \\\hline
  \citet{borkar2025ode} & \dashcheckmark & & \dashcheckmark & Distance to a point  & \cmark  \\\hline
  \citet{lauand2024revisiting} & \dashcheckmark & & \cmark & Distance to a point  & \cmark  \\\hline
  \makecell{Theorem II.1 of \\ \citet{meyn2024bellmen}} & \dashcheckmark & &  & Distance to a set  & \cmark  \\\hline
  \makecell{Hypothetical \\ Combination of  \\ \citet{meyn2024bellmen}  \\ \citet{lauand2024revisiting}} & \dashcheckmark & & \cmark  & Distance to a set  & \cmark  \\\hline
  \citet{liu2025linearq} &  &  & \makecell{\cmark \\ ($p=2$)} & $\E[\norm{w_t}^2]$ &  \cmark \\\hline
  This work & \cmark & \cmark & \cmark & Distance to a set & \cmark  \\\hline
  \end{tabular}
  \caption{\label{table:sa} 
  Comparison with prior works about stochastic approximation with time-inhomogeneous Markovian noise.
  ``a.s.'' is checked if an almost sure convergence is established.
  ``$\Pr(\dots) \geq 1 - \delta$'' is checked if high-probability convergence with sub-Weibull tails is established.
  ``$L^p$'' is checked if $L^p$ convergence is established for all $p \geq 2$, unless otherwise noted.
  ``\cmark'' means the convergence rate is established for the corresponding mode of convergence, while ``\dashcheckmark'' means the convergence is established but the rate is not provided.
  ``$P_{w_t}$'' is checked if the Markov chain $\qty{Y_t}$ is allowed to be time-inhomogeneous.
  } 
\end{table*}
In particular, the focus of this table is stochastic approximation algorithms with non-GAS ODEs and time-inhomogeneous Markovian noise.
To our knowledge,
prior work of this kind is thin and we are only aware of \citet{liu2025linearq} and Theorem II.1 of \citet{meyn2024bellmen}.
We additionally include results about stochastic approximation with GAS ODEs only when they are highly relevant to our results. For instance, we additionally include \citet{qian2026almost} because we rely on their skeleton iterates technique.
We additionally include \citet{borkar2025ode,lauand2024revisiting} because their techniques can be potentially combined with \citet{meyn2024bellmen} to generate results comparable to ours.
Furthermore, regarding specific metric comparisons in Table~\ref{table:sa}, \citet{liu2025linearq} establish the convergence rate of $d\qty(\E[\norm{w_t}^2], [0, B^2])$,
i.e.,
the rate at which the second moment decreases to below some threshold.
By contrast,
this work establishes $L^2$ convergence to a set,\footnote{An $L^2$ rate follows immediately from an $L^p$ rate with any $p > 2$.}
i.e., the convergence rate of $\E[d(w_t, \fBB)^2]$,
which is stronger than \citet{liu2025linearq}.
\begin{remark}
  [Improvement over \citealt{liu2025linearq}]
Define 
\begin{align}
D_t \doteq d\qty(w_t, \fBB(B)) = \max(\norm{w_t} - B, 0) \geq 0.
\end{align}
We have
\begin{align}
\textstyle \max\qty(\norm{w_t}^{2} - B^2, 0) = D_t^{2} + 2BD_t.
\end{align}
Indeed,
when $\norm{w_t} < B$,
both sides are 0.
When $\norm{w_t} \geq B$,
we have
\begin{align}
\textstyle \max\qty(\norm{w_t}^{2} - B^2, 0) = \norm{w_t}^{2} - B^2 = (\norm{w_t} - B)(\norm{w_t} + B) = D_t (D_t + 2B) = D_t^{2} + 2BD_t.
\end{align}
By the convexity of $\max$ and Jensen's inequality, we have
\begin{align}
\textstyle \max\qty(\E[\norm{w_t}^{2}] - B^2,0) \leq \E\qty[\qty(\norm{w_t}^{2} - B^2)^{+}] =  \E[D_t^{2}] + 2B\E[D_t] \leq \E[D_t^{2}] + 2B\sqrt{\E[D_t^{2}]},
\end{align}
which means that the convergence of $\E[d(w_t, \fBB(B))^2]$ implies the convergence of \\$d(\E[\norm{w_t}^2], [0, B^2])$.
Thus, even considering only the case $p = 2$,
this work still constitutes a significant improvement over \citet{liu2025linearq}.
\end{remark}
\begin{proof}
  (Proof idea of Theorems~\ref{thm:rate sa} \&~\ref{thm:concentration sa} and Corollary~\ref{cor:lp sa})
  We rely on the skeleton iterates technique from \citet{qian2026almost}.
  This technique is essentially a significant refinement of the proof technique used in the proof of Proposition 4.8 of \citet{bertsekas1996neuro} and the foundation of this technique is recently formally verified in Lean by \citet{zhang2026towards}.
  Instead of examining $\qty{w_t}$ along the natural time scale $t = 0, 1, \dots$,
  we examine a subsequence $\qty{w_{t_m}}$ indexed by $\qty{t_m}$.
  The subsequence $\qty{t_m}$ is carefully designed such that the magnitude of updates between consecutive times $T_m = \sum_{t=t_m}^{t_{m+1}-1}\alpha_t$ satisfies the Robbins-Monro condition.
  We are then able to recover the template~\eqref{eq:special rs} by identifying $z_m = \norm{w_{t_m}}^2$.
  In this process,
  we need to use an error decomposition technique from \citet{liu2025linearq} to work with the time-inhomogeneous nature of $\qty{Y_t}$,
  thanks to the stronger notion of Lipschitz continuity in Assumption~\ref{asp:P_lip'}.

  The main challenge here is to bound the bias term, denoted $s_{4,m}$ in Appendix~\ref{proof:sec sa}, which arises because the Markov chain $\qty{Y_t}$ is time-inhomogeneous.
 Our analysis pivots on a key insight: we handle this term by dissecting the proof into two cases, depending on the relative positions of the main timescale $t_m$ and the mixing time timescale $t-\tau_{\alpha_t}$.
 For the case where $t_m \geq t-\tau_{\alpha_t}$, we leverage the geometric mixing property of the Markov chain directly. For the case where $t_m < t-\tau_{\alpha_t}$, we employ the tower rule to shift the conditional expectation, allowing us to control the error. 
 In both scenarios, this meticulous case analysis successfully demonstrates that the bias is bounded within $\fO(T_m^2)$.\footnote{For two non-negative sequences $\qty{a_m}$ and $\qty{b_m}$, we say $a_m = \fO(b_m)$ if there exist some constant $C>0$ and an integer $m_0$ such that $a_m \leq C b_m$ for all $m \geq m_0$.}
 This crucial step ensures that the total error is properly controlled.
  
  After recovering~\eqref{eq:special rs},
  all the rest is to verify assumptions of Theorems~\ref{thm:rate} \& \ref{thm:concentration rs} and Corollary~\ref{cor:lp}.
  In particular, we verify the assumption on the growth (Theorem~\ref{thm:bound} (A2)) below.
  \begin{lemma}
\label{lem:bound diff}
There exists some deterministic constants $m_0$ and $C_\tref{lem:bound diff}$, such that for all $m \geq m_0$,
  \begin{align}
    \abs{z_{m+1} - z_m} \leq C_\tref{lem:bound diff}T_m (z_m + 1).
  \end{align}
\end{lemma}
\begin{proof}
    \begin{align*}
        \abs{\norm{w_{t_{m+1}}}^2 - \norm{w_{t_m}}^2} =& \abs{\inner{w_{t_{m+1}} - w_{t_m}}{w_{t_{m+1}} + w_{t_m}}}\\
        \leq & \norm{w_{t_{m+1}} - w_{t_m}} \times \norm{w_{t_{m+1}} + w_{t_m}} \\
        \leq & 2T_m C_\text{\ref{lem:diff_w}}(\norm{w_{t_m}}+1) \times \qty(2 \norm{w_{t_m}} + 2T_m C_\text{\ref{lem:diff_w}}(\norm{w_{t_m}}+1)) \explain{By \eqref{cor:diff_w}}\\
        \leq& 4T_m C_\text{\ref{lem:diff_w}} (1+ T_m C_\text{\ref{lem:diff_w}}) (\norm{w_{t_m}}^2 + 1).
    \end{align*}
\end{proof}
Let $C_\tref{lem:bound diff} \doteq 4C_\text{\ref{lem:diff_w}} (1+ T_0 C_\text{\ref{lem:diff_w}})$. As promised before,
this affine growth assumption is indeed easy to verify.
We recall that a full proof is provided in Appendix~\ref{proof:sec sa}.
\end{proof}

\section{Applications in Reinforcement Learning}
\label{sec:linear q}
In this section,
we apply the results in Section~\ref{sec:sa} to derive new convergence results of reinforcement learning (RL, \citealt{sutton2018reinforcement}) algorithms.
We focus on linear $Q$-learning,
though our results can also be applied to linear SARSA \citep{rummery1994line} to advance the understanding of the long-standing ``chattering'' problem of linear SARSA \citep{gordon1996chattering}.

$Q$-learning \citep{watkins1992q} is one of the most celebrated RL algorithms.
When combined with linear function approximation,
the resulting algorithm, linear $Q$-learning,
is largely believed to suffer from possible instability for decades,
i.e.,
the iterates from linear $Q$-learning can possibly diverge to infinity.
This issue exhibits the deadly triad in RL \citep{sutton2018reinforcement,zhang2022thesis}.
Numerous efforts have been made in prior works to tweak the linear $Q$-learning algorithm to achieve stability. 
See Section~\ref{sec:related_work} and Table 1 of \citet{liu2025linearq} for a more detailed survey.
However, recent works \citep{meyn2024bellmen,liu2025linearq} confirm that linear $Q$-learning is guaranteed to be stable when paired with an $\epsilon$-softmax behavior policy with an adaptive temperature,
without any algorithm tweaks or strong assumptions.
This work builds upon and improves over \citet{meyn2024bellmen,liu2025linearq}.
We now present our new convergence results for linear $Q$-learning.

\tb{Problem Setup.} We consider an infinite horizon Markov Decision Process (MDP, \citealt{bellman1957markovian,puterman2014markov}) with a finite state space $\mathcal{S}$, a finite action space $\mathcal{A}$, a reward function $r: \mathcal{S} \times \mathcal{A} \to \mathbb{R}$, a transition function $p: \fS \times \mathcal{S} \times \mathcal{A} \to [0, 1]$, an initial distribution $p_0 : \fS \to [0, 1]$,
and a discount factor $\gamma \in [0, 1)$.
At time $0$,
an initial state $S_0$ is sampled from $p_0$.
At time $t$,
an agent at a state $S_t$ takes an action sampled from some policy $\pi: \fA \times \fS \to [0, 1]$.

The action-value function $q_\pi(s, a)$ represents the expected discounted reward starting from state $s$, taking action $a$, and following policy $\pi$ thereafter
\begin{align}
    \textstyle q_\pi(s, a) = \E_\pi\qty[\sum_{i=0}^\infty \gamma^i R_{t+i+1} \mid S_t = s, A_t = a ].
\end{align}
The optimal action-value function is defined as $q_*(s, a) = \max_\pi q_\pi(s, a)$, which characterizes the maximum achievable value for each state-action pair. Theoretically, $q^*$ is the unique solution to the Bellman optimality equation. The $Q$-learning algorithm is fundamentally an off-policy method designed to find this solution, meaning it can learn the target policy while using a different, exploratory behavior policy $\pi$ to gather data.
At each time step, the behavior policy $\pi$ provides an action $A_t \sim \pi(\cdot | S_t)$, and the collected tuple $(S_t, A_t, R_{t+1}, S_{t+1})$ is used for learning.  

In linear $Q$-learning, the optimal action-value function $q^*$ is approximated using a linear parameterization $q^*(s, a) \approx x(s, a)^\top w$,
where $x: \mathcal{S} \times \mathcal{A} \to \mathbb{R}^d$ is a feature mapping and $w \in \mathbb{R}^d$ is the parameter vector. The update rule for linear $Q$-learning is given by
\begin{align}
\label{eq:linear q}
\tag{linear $Q$-learning}
    &A_t \sim \mu_{w_t}(\cdot | S_t),\\
    &w_{t+1} = w_t + \alpha_t(R_{t+1} + \gamma \max_a x(S_{t+1}, a)^\top w_t - x(S_t, A_t)^\top w_t)x(S_t, A_t).
\end{align}
Here, $\qty{\alpha_t}$ are learning rates and $\mu_w$ is the behavior policy
\begin{equation}
\label{eq:mu_linear}
    \textstyle \mu_w(a|s) = \frac{\epsilon}{\na} + (1-\epsilon)\frac{\exp(\kappa_w x(s,a)^\top w)}{\sum_b \exp(\kappa_w x(s,b)^\top w)},
\end{equation}
where $\epsilon \in (0, 1)$ sets a baseline level of uniform exploration. The adaptive temperature, $\kappa_w \doteq \frac{\kappa_0}{\max(1, \norm{w})}$ (for some constant $\kappa_0 > 0$), further influences the policy's exploration behavior by adjusting the greediness of the softmax component.
This is exactly the tamed Gibbs policy introduced by \citet{meyn2024bellmen} and is also used later on by \citet{liu2025linearq}.

The instability of linear $Q$-learning (the ``deadly triad'') is primarily driven by a positive feedback loop: as the parameter norm $\norm{w}$ grows due to approximation errors, standard policies become overly greedy, exploiting these errors and pushing $\norm{w}$ to infinity. To solve this, \citet{meyn2024bellmen} introduces the tamed Gibbs policy. By scaling the temperature inversely with $\norm{w}$, it prevents the policy from collapsing into pure exploitation. This acts as a ``restoring force'' that safely maintains exploration at large parameter scales. Based on our theoretical analysis, the design of a stabilizing behavior policy should generally satisfy two core requirements: (1) it must adaptively align with $w$ to induce a mean-field inward drift at large scales (as captured by our Assumption~\ref{asp:bound inner}), and (2) it must maintain global Lipschitz continuity (Assumption~\ref{asp:P_lip'}) and sufficient exploration (Assumption~\ref{asp:markov chain}) rather than degenerating into a discontinuous hard-max. We believe any search for new suitable behavior policies should follow this direction to ensure the stability of linear $Q$-learning. For example, we conjecture that replacing softmax with sparsemax or Gaussian-perturbed greedy selection, while retaining the same inverse-norm scaling and uniform exploration, would likewise stabilize linear $Q$-learning under suitable parameter choices.

\begin{assumption}
\label{asp:markov q}
  The Markov chain $\qty{S_t}$ induced by a uniformly random behavior policy is irreducible and aperiodic.
\end{assumption}
Here, a uniformly random behavior policy refers to the policy that selects each action with equal probability, i.e., $\pi (a|s) = 1/ \na$ for all $s\in \fS$ and $a\in \fA$. This assumption ensures the Markov chain induced by the tamed Gibbs policy $\mu_w$  is also irreducible and aperiodic, since $\mu_w$ explicitly incorporates a baseline level of uniform exploration via its $\epsilon / \na$ term.

\begin{assumption}
\label{asp:full_rank}
  Let $X \in \mathbb{R}^{\ns \na \times d}$ be the feature matrix whose rows are given by $x(s, a)^\top$ for all $(s, a) \in \fS \times \fA$. We assume that $X$ has full column rank, i.e., $\rank(X) = d$.
\end{assumption}

\begin{theorem}
[Convergence of Linear $Q$-Learning]
\label{thm:all q learning}
Let Assumptions~\ref{asp:markov q} and \ref{asp:full_rank} hold. There exist some $\epsilon_0>0$ and $\kappa_1$, such that for any $\epsilon<\epsilon_0$ and $\kappa_0 > \kappa_1$,
if Assumption~\ref{asp:lr1} holds,
then the statements about $\qty{w_t}$ in Theorem~\ref{thm:rate sa} hold for $\qty{w_t}$ generated by~\eqref{eq:linear q}.
If Assumption~\ref{asp:lr2} holds,
then the statements about $\qty{w_t}$ in Theorem~\ref{thm:concentration sa} and Corollary~\ref{cor:lp sa} hold for $\qty{w_t}$ generated by~\eqref{eq:linear q}.
\end{theorem}
We remark that similar to prior works establishing the stability of linear $Q$-learning \citep{meyn2024bellmen, liu2025linearq}, our analysis is restricted to settings with finite state and action spaces. 
Table~\ref{table:linear q} contextualizes the contribution of Theorem~\ref{thm:all q learning} over prior works,
which is essentially inherited from Table~\ref{table:sa}.
To make it more clear,
using the notations in Theorem~\ref{thm:rate sa},
\citet{meyn2024bellmen} shows that $\lim_{t\to\infty} d(w_t, \fBB\qty(B_\tref{thm:rate sa})) = 0$ a.s. but with a broader choice of learning rates,
including $\nu \in (1/2, 1]$.
The hypothetical combination of \citet{meyn2024bellmen} and \citet{lauand2024revisiting} can further broaden the learning rates to include $\nu \in (0, 1]$ but still does not have almost sure convergence rates.
Neither \citet{meyn2024bellmen} nor the hypothetical combination 
provides high-probability concentration bounds with sub-Weibull tails (cf. Theorem~\ref{thm:concentration sa}). 
For $L^p$ convergence rates,
the hypothetical combination can provide better convergence rates and allow broader learning rates than Corollary~\ref{cor:lp sa}.
We close this section with the following proof of Theorem~\ref{thm:all q learning}.
\begin{table*}[t!]
    \centering
  \begin{tabular}{c|c|c|c|c}
  \hline
  & a.s. &  $\Pr(\dots) \geq 1 - \delta$ & $L^p$ & Metric
  \\ \hline
  \citet{meyn2024bellmen} & \dashcheckmark & &  & Distance to a set    \\\hline
  \makecell{Hypothetical \\ Combination of  \\ \citet{meyn2024bellmen}  \\ \citet{lauand2024revisiting}} & \dashcheckmark & & \cmark  & Distance to a set  \\\hline
  \citet{liu2025linearq} &  &  & \makecell{\cmark \\ ($p=2$)} & $\E[\norm{w_t}^2]$ \\\hline
  This work & \cmark & \cmark & \cmark & Distance to a set  \\\hline
  \end{tabular}
  \caption{\label{table:linear q} Comparison with prior works about the convergence of linear $Q$-learning without algorithm tweaks or strong assumptions. 
  ``a.s.'' is checked if an almost sure convergence is established.
  ``$\Pr(\dots) \geq 1 - \delta$'' is checked if convergence with a high probability with sub-Weibull tails is established.
  ``$L^p$'' is checked if $L^p$ convergence is established for all $p \geq 2$, unless otherwise noted.
  ``\cmark'' means the convergence rate is established for the corresponding mode of convergence, while ``\dashcheckmark'' means the convergence is established but the rate is not provided.
  } 
\end{table*}
\begin{proof}
  We proceed by fitting~\eqref{eq:linear q} into~\eqref{eq:sa update} and then verifying the assumptions of Theorems~\ref{thm:rate sa} \&~\ref{thm:concentration sa} and Corollary~\ref{cor:lp sa}.
  Notably,
  the \eqref{eq:linear q} we study is exactly the same as that of \citet{meyn2024bellmen,liu2025linearq} so we can reuse many of their results.
  First, by defining
\begin{align}
Y_{t+1} \doteq& (S_t, A_t, S_{t+1}), \, y \doteq (s, a, s'), \\
\fY \doteq& \{(s \in \mathcal{S}, a \in \mathcal{A}, s' \in \mathcal{S}) \mid p(s'|s,a) > 0\}, \\
H(w,y) \doteq& (r(s,a) + \gamma \max_{a'} x(s',a')^\top w - x(s,a)^\top w)x(s,a),
\end{align}
we are able to fit~\eqref{eq:linear q} into~\eqref{eq:sa update}.
With the identification above, Assumptions~\ref{asp:markov chain} - \ref{asp:bound inner} are already verified in Section 5.2 of \citet{liu2025linearq},
which immediately completes the proof.
In other words, 
we are able to achieve stronger results (thanks to our extensions of the Robbins-Siegmund theorem) with the same set of assumptions.
In particular, 
Assumption~\ref{asp:markov chain} is immediately implied by Assumption~\ref{asp:markov q} (see Section~B.2 of \citealt{zhang2022global} \& Section~5.2 of \citealt{liu2025linearq}).
Thus, as promised,
Assumption~\ref{asp:markov chain} can indeed be trivially satisfied in RL applications.
Assumption~\ref{asp:bound inner} is verified by Lemma~7 in \cite{liu2025linearq},
which originates from Lemma A.9 of \citet{meyn2024bellmen}.
\end{proof}

\section{Related Work}
\label{sec:related_work}
This section reviews related work on learning rates,
convergence guarantees, and linear $Q$-learning.

\tb{Learning rates.} 
In this work,
we always assume that the learning rates are square-summable and diminishing.
There is rich literature on stochastic approximation and RL where the learning rates are not square-summable but still diminishing.
For instance, \citet{moulines2011non} provide a non-asymptotic analysis showing that slower-decaying, non-square-summable learning rates of the form $\mathcal{O}(k^{-\alpha})$ ($\alpha \in (0, 1)$), when coupled with Polyak-Ruppert averaging, robustly achieve optimal convergence rates for minimizing convex objective functions in both finite and infinite-dimensional settings. 
\citet{toulis2017asymptotic} further investigate the asymptotic and finite-sample properties of stochastic gradients by introducing implicit stochastic gradient descent. They demonstrate that implicit updates combined with trajectory averaging maintain numerical stability and optimal statistical efficiency even when employing large learning rates with non-summable variances. 
Extending this to nonparametric least-squares regression, \citet{dieuleveut2016nonparametric} prove that strictly non-square-summable step-sizes (such as constant or very slowly decaying rates) can still attain statistically optimal non-asymptotic convergence rates when trajectory averaging is applied.
Building on these foundational insights, recent literature continues to expand the range of admissible step-size schedules. For instance, \citet{chandak2023concentration} specifically focus on contractive stochastic approximation and RL algorithms, deriving high-probability concentration bounds for strictly non-square-summable step sizes without invoking $\sum \alpha_n^2 < \infty$.
Furthermore, \citet{lauand2024revisiting} recently revisit standard step-size assumptions, demonstrating that the strict square-summability condition can be relaxed while still guaranteeing $L^p$ convergence under time-inhomogeneous Markovian noise. Concurrently, a growing body of work has established fine-grained, non-asymptotic concentration bounds and central limit theorems for Polyak-Ruppert averaged stochastic approximation and $Q$-learning under slowly decaying, non-square-summable step-sizes \citep{mou2020linear, li2023statistical, mou2024optimal, samsonov2024gaussian, khodadadian2025general}.
One important question is whether we can analyze~\eqref{eq:special rs} with techniques from those works.
Our evaluation is negative.
Let $T_n = (n+1)^{-1/2}$ and $A_n \doteq (n+1)^{1/2}$ in Example~\ref{eg:main}.
It can be similarly verified that $\qty{z_n}$ in Example~\ref{eg:main} can be unbounded almost surely.
This means that 
at least we cannot directly invoke those works to analyze~\eqref{eq:special rs} since they typically conclude that the iterates converge to a single point, which conflicts with our new version of~Example~\ref{eg:main}.
Can we extend those works to allow convergence to a bounded set instead of a single point to analyze~\eqref{eq:special rs}?
Our answer is unknown.
We are not excluding this possibility but we are not pursuing this direction in this paper and we are not aware of any existing work that has done this extension.

There is also a rich literature on stochastic approximation and RL with constant learning rates. 
The foundational theories for analyzing the stability and ultimate boundedness of fixed step size algorithms are well-documented in classic textbooks \citep{kushner2003stochastic, borkar2008stochastic}. Building upon these foundations, a significant line of modern research has developed finite-sample guarantees, concentration inequalities, and geometric convergence rates to a neighborhood of the solution for constant step-size stochastic approximation and RL \citep{chen2020finite, bhandari2021finite, durmus2021tight}. 
More recently, there has been substantial progress on the stability analysis of stochastic approximation via refined extensions of the ODE method. 
For example, \citet{borkar2025ode} establish asymptotic statistics, central limit theorems, and \(L^p\) ultimate boundedness for stochastic approximation under diminishing step-sizes. 
Although this work does not study the constant step-size regime directly, it provides a useful ODE-based perspective that has informed subsequent analyses beyond the classical setting. 
In the constant step-size regime, \citet{lauand2025case} then establish geometric ergodicity and explicit \(L^p\) bounds for fixed step-sizes. 
Furthermore, recent algorithm-specific analyses for linear TD and $Q$-learning have explicitly characterized the exact $\mathcal{O}(\alpha)$ stationary bias and distributional convergence induced by constant step-sizes \citep{allmeier2024computing, zhang2024constant, huo2026bias}. While these are important works, they primarily characterize stationary fluctuations or ultimate boundedness around a fixed point, which is fundamentally distinct from our focus on establishing exact convergence to a bounded set under diminishing and square-summable learning rates.

\tb{Convergence rates.}
The study of different modes of convergence of stochastic approximations is an active research area and has a rich literature. 
However, most existing works study the convergence to a fixed point, which can be broadly categorized into three main lines of research. 
The first line focuses on the asymptotic properties and almost sure convergence rates of stochastic approximations under various noise conditions. These works establish precise limiting behaviors, such as the Law of the Iterated Logarithm and almost sure rates via martingale and Lyapunov methods, often accommodating unbounded variance or Markovian noise \citep{chong1999noise, TADIC2002455, koval2003law, tadic2004almost, kouritzin2015convergence, vidyasagar2023convergence, karandikar2024convergence}.
The second line establishes non-asymptotic, finite-sample guarantees for general stochastic approximations with contractive operators. This literature derives tight high-probability concentration bounds, mean-square error rates, and variance-reduced optimal bounds for both constant and decaying step-sizes \citep{dalal2018finitesample, srikant2019finite, thoppe2019concentration, durmus2021tight, chandak2022concentration, mou2022optimal, chen2025concentration}.
The third line specifically investigates the finite-time analysis and optimal sample complexity of classical RL algorithms. These studies provide fine-grained statistical analyses for TD learning and $Q$-learning, tightening the dependencies on the state-action space and the effective horizon under both generative models and Markovian observation settings \citep{szepesvari1997asymptotic, even2003learning, korda2015td, dalal2018finite, qu2020finite, prashanth2021concentration, li2024q, chandak2023concentration}.
Crucially, while these works provide profound insights into the convergence rates to a unique optimal solution, they are not directly applicable to the regime studied in this paper, where the iterates inherently chatter and only converge to a bounded set.

\tb{Linear $Q$-learning.}
The potential divergence of linear $Q$-learning has driven numerous efforts to stabilize the algorithm, typically through explicit algorithmic modifications or restrictive assumptions. Common algorithmic interventions include utilizing target networks \citep{mnih2015human,zhang2021breaking} and replay buffers \citep{lin1992self}, employing projection operators \citep{melo2008analysis, carvalho2020new}, or adding ridge regularization \citep{lim2024regularized}. 
When analyzing the unmodified algorithm, prior works often impose strong assumptions on the features, data sampling, or the behavior policy. For example, some analyses require orthogonal and binary features \citep{lee2019unified}, or assume transitions are sampled i.i.d.\ from a fixed distribution \citep{lee2019unified, carvalho2020new, lim2024regularized} or from the varying stationary distribution of the current $\epsilon$-greedy policy \citep{gopalan2026dqn}. Another common approach is to impose restrictive matrix definiteness or expectation bounds on the behavior policy \citep{melo2008analysis, chen2022finite}, which essentially require the behavior policy to be arbitrarily close to the optimal policy. Similar stringent conditions, often coupled with the restrictive Bellman completeness assumption, are also prevalent in analyses of $Q$-learning with neural networks \citep{xu2020finite, cai2023neural}.

Despite these successful algorithmic interventions, it is crucial to recognize the inherent nature of the unmodified algorithm. As thoroughly discussed by \citet{meyn2024bellmen}, standard linear $Q$-learning is fundamentally unstable due to the lack of a GAS ODE. 
This fundamental instability highlights why convergence to a single fixed point is generally unattainable without the aforementioned strong assumptions or algorithmic tweaks, making the convergence to a bounded set (as studied in this work and \citealt{meyn2024bellmen,liu2025linearq}) a more natural and fundamental characterization of its dynamics.

\tb{Interplay of noise and step-sizes.}
Recent advancements have also sought to relax standard noise and step-size assumptions from complementary perspectives. For instance, a concurrent work by \citet{nguyen2026almost} (who explicitly note the connection to our results in their Remark 2) investigates the almost sure convergence of SA under heavy-tailed noise possessing only a finite $p$-th moment ($p>1$). From the perspective of the Robbins-Siegmund theorem, their setting essentially modifies the standard almost supermartingale bound from $\mathbb{E}_n[z_{n+1}] \leq (1 - \alpha T_n) z_n + \xi T_n^2$ to $\mathbb{E}_n[z_{n+1}] \leq (1 - \alpha T_n) z_n + \xi T_n^p$. To guarantee almost sure convergence to a unique fixed point, they naturally enforce the $p$-th power summability of the step-sizes (i.e., $\sum T_n^p < \infty$). This provides a fascinating orthogonal complement to our work. While \citet{nguyen2026almost} adjust the step-size summability condition to counteract the $\mathcal{O}(T_n^p)$ noise term and preserve point convergence, our framework tackles the regime where the zero-order term inherently scales with $T_n$, yielding $\mathbb{E}_n[z_{n+1}] \leq (1 - \alpha T_n) z_n + \xi T_n$ as shown in \eqref{eq:special rs}. Since the Robbins-Monro condition strictly requires $\sum T_n = \infty$, the summability of the zero-order term is fundamentally broken. However, by maintaining the standard square-summability of the step-sizes ($\sum T_n^2 < \infty$), our extension uniquely establishes exact convergence to a bounded set, which is critical for inherently oscillatory algorithms like linear $Q$-learning.

\section{Conclusion}
\label{sec:conclusion}
We extend the classical Robbins-Siegmund theorem to handle the critical regime where the zero-order term has a divergent sum but is square-summable, 
a setting that arises naturally in many modern stochastic approximation and RL algorithms. Our extended framework provides precise almost sure convergence rates, 
high-probability concentration bounds,
and $L^p$ convergence rates,
enabling fine-grained analysis of algorithms previously beyond theoretical reach.

As a key application, we establish the first complete convergence characterization of linear $Q$-learning, including almost sure convergence rates, maximal concentration bound with sub-Weibull tails, and $L^p$ convergence rates,
demonstrating that linear $Q$-learning converges with explicit, quantifiable rates despite having been deemed unstable for decades.

Our framework opens several promising directions: extending the analysis to nonlinear function approximation, handling non-Markovian noise models, and developing adaptive algorithms that exploit the precise convergence characterization to optimize learning rates. More broadly, our results suggest that many algorithms previously thought to be theoretically intractable may yield to careful analysis using our extended framework.

\acks{This work is supported in part by the US National Science Foundation under the awards III-2128019, SLES-2331904, and CAREER-2442098, the Commonwealth Cyber Initiative's Central Virginia Node under the award VV-1Q26-001, VV-1Q27-007, and VV-1Q27005, 4VA at UVA, a Cisco Faculty Research Award, and a Google Research Award. }

\appendix
\section{Proof of Section~\ref{sec:special}}
This appendix proves Theorem~\ref{thm:rs general}.

\subsection{Proof of Theorem~\ref{thm:rs general}}
\label{sec:proof rs general}
\begin{proof}
    Let $u_n \doteq (z_n - B_\tref{thm:rs general})^+$. Again we consider two cases.
    \tb{(1) When $u_n > 0$.} Then $u_n = z_n - B_\tref{thm:rs general}$, Theorem~\ref{thm:rs general} (A3) gives $x_n - y_n \leq - c_n u_n$. Similar to \eqref{eq:decomp_u}, we have $u_{n+1}^2 \leq \qty(z_{n+1} - z_n + u_n )^2$. 
    Thus
    \begin{align}
        \textstyle &\E_n\qty[u_{n+1}^2 ]  \\
        \leq &\textstyle \E_n\qty[u_n^2 + 2u_n (z_{n+1} - z_n) + (z_{n+1} - z_n)^2] \\
        \leq &\textstyle u_n^2 + 2u_n \E_n\qty[z_{n+1} - z_n] + \E_n\qty[b_n^2(z_n + 1)^2] \explain{By Theorem~\ref{thm:rs general} (A2)}\\
        \leq &\textstyle u_n^2 + 2u_n (a_n z_n + x_n - y_n ) + b_n^2(u_n + 1 + B_\tref{thm:rs general} )^2 \explain{By \eqref{eq:rs}} \\
        \leq &\textstyle u_n^2 + 2u_n (a_n (u_n + B_\tref{thm:rs general} ) - c_n u_n)+ 2b_n^2u_n^2 + 2(1 + B_\tref{thm:rs general} )^2 b_n^2 \explain{Since $u_n>0$ and $x_n - y_n\leq -c_nu_n$} \\
        = &\textstyle u_n^2 + 2a_n u_n^2 + 2a_n B_\tref{thm:rs general}  u_n - 2c_n u_n^2 + 2b_n^2 u_n^2 + 2(1 + B_\tref{thm:rs general} )^2 b_n^2  \\
        \leq &\textstyle u_n^2 + 2a_n u_n^2 + a_n B_\tref{thm:rs general}  (u_n^2 + 1) - 2c_n u_n^2 + 2b_n^2u_n^2 + 2(1 + B_\tref{thm:rs general} )^2 b_n^2 \\
        = &\textstyle u_n^2 + (2+ B_\tref{thm:rs general})a_n u_n^2 + a_n B_\tref{thm:rs general} - 2c_n u_n^2  + 2b_n^2u_n^2 + 2(1 + B_\tref{thm:rs general} )^2 b_n^2 \\
        = &\textstyle (1 + 2b_n^2 + (2+ B_\tref{thm:rs general})a_n )u_n^2 + 2(1 + B_\tref{thm:rs general} )^2 b_n^2 + a_n B_\tref{thm:rs general} - 2c_n u_n^2. 
    \end{align}
    \tb{(2) When $u_n = 0$.} We have $z_n \leq B_\tref{thm:rs general}$. By \eqref{eq:+ diff}, we obtain 
    \begin{align}
        \textstyle\abs{u_{n+1}} = \abs{u_{n+1} - u_n} \leq \abs{(z_{n+1} - B_\tref{thm:rs general}) - (z_n - B_\tref{thm:rs general})} = \abs{z_{n+1} - z_n}.
    \end{align}
    Thus we have
    \begin{align}
        \textstyle\E_n\qty[u_{n+1}^2] \leq&\textstyle \E_n\qty[(z_{n+1} - z_n)^2] \\
        \leq&\textstyle b_n^2 \qty(z_n  + 1)^2 \explain{By Theorem~\ref{thm:rs general} (A2)}\\
        \leq&\textstyle b_n^2 (B_\tref{thm:rs general} + 1)^2 \explain{Since $z_n\in[0, B_\tref{thm:rs general}]$}\\
        \leq&\textstyle (1 + 2b_n^2 + (2+ B_\tref{thm:rs general})a_n )u_n^2 + 2(1 + B_\tref{thm:rs general} )^2 b_n^2 + a_n B_\tref{thm:rs general} - 2c_n u_n^2. \explain{Since $u_n =0 $}
    \end{align}
    Combining the two cases, there exist non-negative sequences
\begin{align}
  \textstyle d_n \doteq 2b_n^2 + (2+ B_\tref{thm:rs general})a_n,  e_n \doteq 2(1+ B_\tref{thm:rs general})^2 b_n^2 + B_\tref{thm:rs general} a_n,
\end{align}
such that
\begin{align}
  \textstyle \E_n\qty[u_{n+1}^2] \leq (1 + d_n) u_n^2 + e_n - 2c_n u_n^2.
\end{align}
    Since $\sum_n d_n < \infty$ and $\sum_n e_n < \infty$ a.s., the
    Robbins--Siegmund theorem (Lemma~\ref{thm:rs}) yields that $u_n^2$ converges almost
    surely to a finite limit. Hence $(z_n -  B_\tref{thm:rs general})^+$ has an a.s. finite limit. From Lemma~\ref{thm:rs}, we also derive $\sum_n 2c_n u_n^2 < \infty$. 
    Given $\sum_n c_n = \infty$ a.s., following the same argument in Theorem~\ref{thm:bound} then completes the proof.
\end{proof}

\section{Proof of Section~\ref{sec:as con}}
This appendix provides the proofs of the nonasymptotic results in Section~\ref{sec:as con}.
\subsection{Proof of Theorem~\ref{thm:rate}}
\label{proof:sec rate}
\begin{proof}
  Recall we derived the following bound in \eqref{eq:almost_mrtg}
  \begin{align}
    \textstyle \E_n\qty[u_{n+1}^2] \leq \qty(1-2\alpha T_n + 2B_{\tref{thm:bound},1}^2T_n^2)u_n^2 + 2B_{\tref{thm:bound},1}^2\qty(1 + \frac{\xi }{\alpha })^2T_n^2,
  \end{align}  
  where $u_n = \qty(z_n - \frac{\xi}{\alpha})^+$. We apply Lemma~\ref{aux:thm2} with the identification
\begin{equation}
    \textstyle z_n \leftrightarrow u_n^2, a_n \leftrightarrow 2B_{\tref{thm:bound},1}^2T_n^2,
     x_n \leftrightarrow 2B_{\tref{thm:bound},1}^2\qty(1 + \frac{\xi }{\alpha })^2T_n^2, 
     b_n \leftrightarrow 2\alpha T_n.
\end{equation}
We verify the assumptions (A1)–(A3) of Lemma~\ref{aux:thm2}.

\tb{(A1) of Lemma~\ref{aux:thm2}.} Theorem~\ref{thm:rate} (A1) gives $\sum_n T_n = \infty$, thus $\sum_n b_n = \infty$;  Theorem~\ref{thm:rate} (A2) gives $\sum_n T_n^2 < \infty$, thus $\sum_n a_n < \infty$ and $\sum_n x_n < \infty$. These trivially confirm that Lemma~\ref{aux:thm2} (A1) holds. 

\tb{(A2) of Lemma~\ref{aux:thm2}. }From Theorem~\ref{thm:rate} (A1), there exists some $n_1$ such that for all $n \geq n_1$, we then have $nT_n > \frac{\eta}{2\alpha }$. That is, $b_n=2\alpha T_n \geq \frac{\eta}{n}$ for $n \geq n_1$.

\tb{(A3) of Lemma~\ref{aux:thm2}. }The first part is satisfied by Theorem~\ref{thm:rate} (A2), we now validate the second part.
Theorem~\ref{thm:rate} (A1) implies there exist some constant $\eta_0>0$, such that $\frac{\eta_0}{2\alpha}\leq \liminf_n nT_n - \frac{\eta}{2\alpha}$. Thus there exists some $n_1$ such that for all $n \geq n_1$, 
we have $ \eta_0\leq  2\alpha n T_n - \eta $. That is,
\begin{align}
    \textstyle \sum_{n=1}^{\infty} \left(b_n - \frac{\eta}{n}\right) 
    =&\textstyle  \sum_{n=1}^{\infty} \left(\frac{n \cdot 2\alpha T_n - \eta}{n}\right)\\
    =&\textstyle  \sum_{n=1}^{n_1} \left(2\alpha T_n - \frac{\eta}{n}\right) + \sum_{n=n_1+1}^{\infty} \left(\frac{n \cdot 2\alpha T_n - \eta}{n}\right)\\
    \geq &\textstyle  \sum_{n=1}^{n_1} \left(2\alpha T_n - \frac{\eta}{n}\right) + \sum_{n=n_1+1}^{\infty} \frac{\eta_0}{n}.
\end{align}
Here the first summation is finite, while the second summation converges to $\infty$. 
Thus (A3) of Lemma~\ref{aux:thm2} holds.

By Lemma~\ref{aux:thm2}, we obtain $\lim_{n \to \infty} n^{\eta}u_n^2 = 0$ a.s., which implies
\begin{align}
    \textstyle\lim_{n \to \infty} n^{\eta} u_n^2 =\lim_{n \to \infty} n^{\eta}\qty((z_n - B_\tref{thm:bound})^+)^2 =\lim_{n \to \infty} n^{\eta}d\qty(z_n, [0, B_\tref{thm:bound}])^2 = 0 \qq{a.s.}
\end{align}
Taking square roots then completes the proof. 
\end{proof}

\subsection{Proof of Theorem~\ref{thm:concentration rs}}
\label{sec:proof concentration rs}
\begin{proof}
To establish the concentration bound, we adopt the inductive (bootstrapping / recursive) framework from \citet{chen2025concentration}. Our argument begins with the establishment of the inductive basis.
\begin{lemma} \label{lem:inductive basis}
There exist some deterministic integers $C_\tref{lem:inductive basis}$ and $C'_\tref{lem:inductive basis}$ such that for any $n \geq 0$,
\begin{equation}
    \qty((z_n - B_\tref{thm:bound})^+)^2 \leq C'_\tref{lem:inductive basis}(n+1)^{C_\tref{lem:inductive basis}} \quad a.s.
\end{equation}
\end{lemma}
The proof is in Appendix~\ref{proof:inductive basis}. 

\begin{lemma} 
\label{lem:inductive step}
There exists some deterministic constant $\bar{n}$ such that the following implication relation holds. If for any $\delta \in (0,1)$, there exists some non-decreasing sequence $\{B_n(\delta)\}_{n=\bar{n}}^{\infty}$ such that
\begin{align}
\label{eq prob ineq 1}
    \Pr\left(\forall n \geq \bar{n}, \qty((z_n - B_\tref{thm:bound})^+)^2 \leq B_n(\delta)\right) \geq 1 - \delta,
\end{align}
then for any $\delta' \in (0, 1-\delta)$, there exists a new sequence $\{B_n(\delta, \delta')\}_{n=\bar{n}}^{\infty}$ such that
\begin{align}
    \Pr\left(\forall n \geq \bar{n}, \qty((z_n - B_\tref{thm:bound})^+)^2 \leq B_n(\delta, \delta')\right) \geq 1 - \delta - \delta',
\end{align}
where $B_n(\delta, \delta') = C_{\tref{lem:inductive step}}T_n\max\qty{B_n(\delta), 1}(1 + \ln(1/\delta') + \ln(n+n_0))$ with $C_{\tref{lem:inductive step}}$ being some constant depending on $\bar{n}$ and independent of $n, \delta, \delta'$.
\end{lemma}
Similar to our previous analysis,
we need to consider two scenarios: $u_n > 0$ and $u_n = 0$. The proof is in Appendix~\ref{proof:inductive step}. 
In essence, each application of Lemma~\ref{lem:inductive step} 
enhances our bound by a factor of $T_n$. Given that 
Lemma~\ref{lem:inductive basis} establishes an almost sure 
bound of order $n^{C_\tref{lem:inductive basis}}$, we can 
achieve the desired bound by applying Lemma~\ref{lem:inductive step} 
exactly $C_\tref{lem:inductive basis}+1$ times. Further applications would be counterproductive, 
as they would violate the non-decreasing property of the bound $B_n$.
\begin{lemma} 
\label{lem:concentration rs}
Under Assumption~\ref{asp:T_n}, there exists some constants $C_\tref{lem:concentration rs} = C_\tref{lem:inductive basis} + 1$ and $C'_\tref{lem:concentration rs}$ such that for any $\delta > 0$, it holds, with probability at least $1 - \delta$, that for all $n \geq 0$,
\begin{align}
    \textstyle\qty((z_n - B_\tref{thm:bound})^+)^2 \leq C'_\tref{lem:concentration rs}\frac{1}{n+n_0}\left[\ln\left(\frac{C_\tref{lem:concentration rs}}{\delta}\right) + 1 + \ln(n+n_0)\right]^{C_\tref{lem:concentration rs}}.
\end{align}
\end{lemma}
The proof follows the same methodology as \citet{qian2026almost}. 
The details are in Appendix~\ref{proof:concentration rs}. 
    
\end{proof}
\subsubsection{Proof of Lemma~\ref{lem:inductive basis}}
\label{proof:inductive basis}
\begin{proof}
For any $n\geq 0$, we have
\begin{align}
    \textstyle \abs{u_{n+1}-u_n} \leq& \textstyle \abs{\qty(z_{n+1}-\frac{\xi}{\alpha}) - \qty(z_n-\frac{\xi}{\alpha})} \explain{By \eqref{eq:+ diff}}\\
    =&\textstyle \abs{z_{n+1}-z_n} \\
    \leq&\textstyle B_{\tref{thm:bound},1}T_n\qty(z_n + 1) \explain{By Theorem~\ref{thm:bound} (A2)}\\
    \leq&\textstyle B_{\tref{thm:bound},1}T_n\qty(u_n + \frac{\xi}{\alpha} + 1) \explain{By definition of $u_n$}.
\end{align}
Hence
\begin{align}
    \textstyle u_{n+1} \leq u_n + B_{\tref{thm:bound},1}T_n\qty(u_n + \frac{\xi}{\alpha} + 1) \leq u_0 +  B_{\tref{thm:bound},1}\sum_{i=0}^n T_i u_i +  2B_{\tref{thm:bound},1}\qty(1 + \frac{\xi }{\alpha })\sum_{i=0}^n T_i.
\end{align}
We first recall a discrete Gronwall inequality.
\begin{lemma}
  [Discrete Gronwall Inequality](Lemma 8 in Section 11.2 of \citealt{borkar2008stochastic})
\label{lem:discrete_gronwall}
    For non-negative real sequences $\qty{x_n}_{n \geq 0}$ and $\qty{a_n}_{n \geq 0}$ and scalars  $C, L \geq 0$,  it holds
    \begin{align}
        \textstyle x_{n+1} \leq C + L\sum_{i=0}^n a_ix_i \quad \forall n
    \implies
        \textstyle x_{n+1} \leq (C+x_0)\exp({L\sum_{i=0}^n a_i}) \quad \forall n.
    \end{align}
\end{lemma}
Applying the above result, we obtain that for all $n \geq 0$,
\begin{align}
    \textstyle u_n \leq \left(u_0 + 2B_{\tref{thm:bound},1}\qty(1 + \frac{\xi}{\alpha})\sum_{i=0}^n T_i\right) \exp\left(1 + B_{\tref{thm:bound},1}\sum_{i=0}^{n-1} T_i \right).
\end{align}
Noting $T_n = \frac{C_\tref{asp:T_n}}{n+n_0}$ and $\sum_{i=0}^n T_i \leq C_\tref{asp:T_n}\ln(n+n_0+1) $ then completes the proof. 
\end{proof}

\subsubsection{Proof of Lemma~\ref{lem:inductive step}}
\label{proof:inductive step}
\begin{proof}
  For any $\delta \in (0, 1)$, consider the non-decreasing sequence $\qty{B_n(\delta)}$ defined in \eqref{eq prob ineq 1}.
For each $n \geq \bar{n}$, where $\bar{n}$ is a sufficiently large constant, 
define an event 
\begin{align}
    E_n(\delta) = \qty{u_k^2 \leq B_k(\delta), \, \forall k = \bar{n}, \bar{n} + 1, \dots, n}.
\end{align}
Notably, $\qty{E_n(\delta)}_{n \geq \bar{n}}$ is by definition a sequence of decreasing events, i.e., $E_{n+1}(\delta) \subset E_n(\delta)$.
In addition, according to~\eqref{eq prob ineq 1}, 
we have $\Pr(E_n(\delta)) \geq 1 - \delta$ for any $n \geq \bar{n}$. 
Let $\lambda_n = \theta T_n^{-1} \qty(\max\qty{B_n(\delta), 1})^{-1}$
where $\theta$ is a tunable constant yet to be chosen.

We now proceed to show that $\qty{\exp \left(\lambda_n \mathbbm{1}_{E_n(\delta)} u_n^2 \right)}$ is ``almost'' a supermartingale.

\tb{(1) When $u_n > 0$.} Then $u_n =z_n - \frac{\xi}{\alpha}$. By \eqref{eq:decomp_u}, we have $u_{n+1}^2 \leq u_n^2  + (z_{n+1} - z_n)^2 +2u_n (z_{n+1} - z_n)$.
By first multiplying $\lambda_{n+1} \mathbbm{1}_{E_{n+1}(\delta)}$ and then taking exponential on both sides of the inequality, 
\begin{align}
    &\textstyle\exp(\lambda_{n+1} \mathbbm{1}_{\E_{n+1}(\delta)}u_{n+1}^2)\\
    \leq&\textstyle \exp(\lambda_{n+1} \mathbbm{1}_{\E_{n+1}(\delta)}u_n^2) \times \exp(\lambda_{n+1} \mathbbm{1}_{\E_{n+1}(\delta)}(z_{n+1} - z_n)^2)\\
    & \times \exp(\lambda_{n+1} \mathbbm{1}_{\E_{n+1}(\delta)}2u_n(z_{n+1} - z_n))\\
    \leq&\textstyle \exp(\lambda_{n+1} \mathbbm{1}_{\E_n(\delta)}u_n^2) \times \exp(\lambda_{n+1} \mathbbm{1}_{\E_n(\delta)}(z_{n+1} - z_n)^2)\\
    & \times \exp(\lambda_{n+1} \mathbbm{1}_{\E_n(\delta)}2u_n(z_{n+1} - z_n)). \explain{Since $E_{n+1}(\delta)\subset E_n(\delta)$}
\end{align}
Taking expectation conditioned on $\mathcal{F}_n$ on both sides of the previous inequality,  we obtain
\begin{align}
    &\textstyle\mathbb{E}_n\qty[\exp(\lambda_{n+1} \mathbbm{1}_{E_{n+1}(\delta)}u_{n+1}^2)]\\
    \leq&\textstyle \exp(\lambda_{n+1} \mathbbm{1}_{E_n(\delta)}u_n^2) \times \exp(\lambda_{n+1} \mathbbm{1}_{E_n(\delta)} 2B_{\tref{thm:bound},1}^2 T_n^2\qty(u_n^2+\qty(1+\frac{\xi }{\alpha })^2)) \\
    &\textstyle\times \E_n\qty[\exp(\lambda_{n+1} \mathbbm{1}_{E_n(\delta)}2u_n(z_{n+1} - z_n))]\label{eq bound exp},
\end{align}
where the second term is obtained by 
\begin{align}
    (z_{n+1} - z_n)^2 \leq& B_{\tref{thm:bound},1}^2 T_n^2(z_n+1)^2 \explain{By Theorem~\ref{thm:bound} (A2)}\\
    =&\textstyle B_{\tref{thm:bound},1}^2 T_n^2\qty(\qty(u_n + \frac{\xi}{\alpha})+1)^2\explain{Since $u_n = z_n - \frac{\xi }{\alpha }$}\\
    \leq&\textstyle 2B_{\tref{thm:bound},1}^2 T_n^2\qty(u_n^2+\qty(1+\frac{\xi }{\alpha })^2).
\end{align}
 To bound the last term, we will use Hoeffding's lemma. 
We have 
\begin{align}
    \textstyle\E_n\qty[\lambda_{n+1} \mathbbm{1}_{E_n(\delta)}2 u_n(z_{n+1} - z_n)] =&\textstyle 2\lambda_{n+1} \mathbbm{1}_{E_n(\delta)} u_n\E_n\qty[(z_{n+1} - z_n)]\\
    \leq&\textstyle 2\lambda_{n+1} \mathbbm{1}_{E_n(\delta)} u_n(-\alpha T_n z_n + \xi T_n). \explain{By \eqref{eq:special rs}}\\
    =&\textstyle -2\lambda_{n+1} \mathbbm{1}_{E_n(\delta)}  \alpha T_n u_n^2. \explain{Since $u_n =z_n - \frac{\xi}{\alpha}$}
\end{align}
Additionally, by (A2) of Theorem~\ref{thm:bound}, we obtain
\begin{align}
    &\textstyle\abs{\lambda_{n+1} \mathbbm{1}_{E_n(\delta)}2 u_n(z_{n+1} - z_n)}\\
    \leq&\textstyle 2\lambda_{n+1} \mathbbm{1}_{E_n(\delta)} B_{\tref{thm:bound},1}T_n u_n \qty(u_n + \frac{\xi}{\alpha} + 1)\\
    \leq&\textstyle 2\lambda_{n+1} \mathbbm{1}_{E_n(\delta)} B_{\tref{thm:bound},1} T_n u_n \qty(B_n(\delta)^{1/2} + 1 + \frac{\xi}{\alpha }) \explain{Definition of $E_n(\delta)$}.
\end{align} 
Now we can apply Hoeffding's lemma to the last term in \eqref{eq bound exp}, which gives us
\begin{align}
    &\textstyle\E_n\left[\exp(\lambda_{n+1} \mathbbm{1}_{E_n(\delta)}2u_n(z_{n+1} - z_n)) \right]\\
    \leq&\textstyle \exp(-2\lambda_{n+1} \mathbbm{1}_{E_n(\delta)} \alpha T_n u_n^2 + 32\lambda_{n+1}^2 \mathbbm{1}_{E_n(\delta)}u_n^2 B_{\tref{thm:bound},1}^2 T_n^2 \qty(B_n(\delta)^{1/2} + 1 + \frac{\xi }{\alpha })^2 / 8) \\
    \leq&\textstyle \exp(-2\lambda_{n+1} \mathbbm{1}_{E_n(\delta)} \alpha T_n u_n^2 + 8\lambda_{n+1}^2 \mathbbm{1}_{E_n(\delta)}u_n^2 B_{\tref{thm:bound},1}^2 T_n^2  \qty(B_n(\delta) + \qty(1 + \frac{\xi }{\alpha })^2)).
\end{align}
Applying this back to \eqref{eq bound exp}, we have 
\begin{align}
    &\textstyle\E_n\qty[\exp(\lambda_{n+1} \mathbbm{1}_{E_{n+1}(\delta)}u_{n+1}^2)]\\
    \leq&\textstyle \exp(\lambda_{n+1} \mathbbm{1}_{E_n(\delta)}u_n^2\qty(1 - 2\alpha T_n + 2B_{\tref{thm:bound},1}^2 T_n^2+ 8\lambda_{n+1}B_{\tref{thm:bound},1}^2T_n^2\qty(B_n(\delta) + \qty(1 + \frac{\xi }{\alpha })^2)))\\
    &\textstyle\times \exp(\lambda_{n+1} \mathbbm{1}_{E_n(\delta)}2B_{\tref{thm:bound},1}^2 \qty(1+\frac{\xi }{\alpha })^2T_n^2).
\end{align}

\tb{(2) When $u_n = 0$.} We obtain
\begin{align}
    &\textstyle\E_n\qty[\exp(\lambda_{n+1} \mathbbm{1}_{E_{n+1}(\delta)}u_{n+1}^2)] \\
    \leq& \textstyle\E_n\qty[\exp(\lambda_{n+1} \mathbbm{1}_{E_n(\delta)}u_{n+1}^2)] \explain{Since $E_{n+1}(\delta)\subset E_n(\delta)$}\\
    =&\textstyle \E_n\qty[\exp(\lambda_{n+1} \mathbbm{1}_{E_n(\delta)} (z_{n+1}-z_n)^2)] \explain{By \eqref{eq:diff_u_s}}\\
    \leq&\textstyle \E_n\qty[\exp(\lambda_{n+1} \mathbbm{1}_{E_n(\delta)} B_{\tref{thm:bound},1}^2\qty(u_n + 1 + \frac{\xi }{\alpha })^2T_n^2)] \explain{By Theorem \ref{thm:bound} (A2) and since $z_n \in [0,\frac{\xi}{\alpha}]$}\\
    =&\textstyle \exp(\lambda_{n+1} \mathbbm{1}_{E_n(\delta)} B_{\tref{thm:bound},1}^2\qty(1 + \frac{\xi}{\alpha})^2T_n^2). \explain{Since $u_n=0$}
\end{align}

To combine two cases and simplify notation, define 
\begin{align}
    &\textstyle D_{1,n} \doteq \frac{\lambda_{n+1}}{\lambda_n}\qty(1 - 2\alpha T_n + 2B_{\tref{thm:bound},1}^2 T_n^2 + 8\lambda_{n+1}B_{\tref{thm:bound},1}^2T_n^2\qty(B_n(\delta) + \qty(1 + \frac{\xi }{\alpha })^2)), \\
    &\textstyle D_{2,n} \doteq \lambda_{n+1} 2B_{\tref{thm:bound},1}^2 \qty(1+\frac{\xi }{\alpha })^2T_n^2.
\end{align}
Then, the previous inequalities (for both cases where $u_n > 0$ or $u_n = 0$) can be written in a unified general form as follows
\begin{align}
  \label{eq exp bound 2}
    &\textstyle\E_n\qty[\exp(\lambda_{n+1} \mathbbm{1}_{E_{n+1}(\delta)}u_{n+1}^2)]
    \leq\textstyle \exp(\qty(D_{1,n})^+ \lambda_n \mathbbm{1}_{E_n(\delta)}u_n^2)\exp(D_{2,n}\mathbbm{1}_{E_n(\delta)}).
\end{align}
Next, we bound the terms $D_{1,n}$ and $D_{2,n}$. Since $\lambda_n = \theta T_n^{-1}\max\qty{B_n(\delta),1}^{-1}$, denote $C_{\tref{lem:inductive step},1} \doteq 1+\frac{\xi }{\alpha }$, then for $n$ sufficiently large we have
\begin{align}
    \textstyle D_{1,n} \mathbbm{1}_{E_n(\delta)} &\textstyle= \frac{\lambda_{n+1}}{\lambda_n}(1 - 2\alpha T_n + 2B_{\tref{thm:bound},1}^2 C_{\tref{lem:inductive step},1}^2T_n^2+ 8\lambda_{n+1}B_{\tref{thm:bound},1}^2T_n^2(B_n(\delta) + C_{\tref{lem:inductive step},1}^2))\mathbbm{1}_{E_n(\delta)}\\
    &\textstyle\leq \frac{\lambda_{n+1}}{\lambda_n}(1 - \alpha T_n + 8\lambda_{n+1}B_{\tref{thm:bound},1}^2T_n^2(B_n(\delta) + C_{\tref{lem:inductive step},1}^2))\mathbbm{1}_{E_n(\delta)}\\
    &\textstyle= \frac{T_n\max\qty{B_n(\delta), 1}}{T_{n+1}\max\qty{B_{n+1}(\delta), 1}}\left(1 - \alpha T_n+\frac{8\theta B_{\tref{thm:bound},1}^2T_n^2(B_n(\delta) + C_{\tref{lem:inductive step},1}^2)}{T_{n+1}\max\qty{B_{n+1}(\delta), 1}}\right)\mathbbm{1}_{E_n(\delta)}\\
    &\textstyle\leq \frac{T_n}{T_{n+1}}\left(1 - \alpha  T_n +\frac{8B_{\tref{thm:bound},1}^2\theta T_n^2}{T_{n+1}}\frac{C_{\tref{lem:inductive step},1}^2 + B_n(\delta)+1}{\max\qty{B_{n+1}(\delta), 1}}\right)\mathbbm{1}_{E_n(\delta)}\\
    &\textstyle\leq \frac{T_n}{T_{n+1}}\left(1 - \alpha  T_n +\frac{8B_{\tref{thm:bound},1}^2\theta T_n^2}{T_{n+1}}(C_{\tref{lem:inductive step},1}^2 + 2)\right)\\
    &\textstyle\leq \frac{T_n}{T_{n+1}}\left(1 - \alpha  T_n +16B_{\tref{thm:bound},1}^2\theta T_n(C_{\tref{lem:inductive step},1}^2 + 2)\right).
\end{align}
Recall $C_\tref{asp:T_n} > 1/\alpha $ (Assumption~\ref{asp:T_n}).
Select any $\kappa \in (1/C_\tref{asp:T_n}, \alpha )$ and define\\
$\theta \doteq \frac{1}{C_{\tref{lem:inductive step},1}^2+2}\min\{1, \frac{\alpha -\kappa}{16B_{\tref{thm:bound},1}^2}\}$.
Then we have
\begin{align}
  \textstyle D_{1,n} \mathbbm{1}_{E_n(\delta)} \leq \frac{T_n}{T_{n+1}}\left(1 - T_n \kappa\right) =\frac{n+n_0+1}{n+n_0}\left(1 - \frac{C_\tref{asp:T_n}}{n+n_0} \kappa\right).
\end{align}
Since $C_\tref{asp:T_n} \kappa > 1$, it is easy to compute that $D_{1,n} \mathbbm{1}_{E_n(\delta)} < 1$ holds for all $n$.
Now, consider $D_{2,n}$. We have by definition of $\lambda_n$ that
\begin{align}
    \textstyle D_{2,n}\mathbbm{1}_{E_n(\delta)} = \frac{2\theta B_{\tref{thm:bound},1}^2C_{\tref{lem:inductive step},1}^2T_n^2}{T_{n+1}\max\qty{B_{n+1}(\delta), 1}}\mathbbm{1}_{E_n(\delta)}
    \leq \frac{2\theta B_{\tref{thm:bound},1}^2C_{\tref{lem:inductive step},1}^2T_n^2}{T_{n+1}}
    \leq 4 B_{\tref{thm:bound},1}^2 C_{\tref{lem:inductive step},1}^2 T_n.
\end{align}
Using the upper bounds we obtained for the terms $D_{1,n}$ and $D_{2,n}$ in \eqref{eq exp bound 2}, 
we have for any $n \geq \bar{n}$, 
\begin{align}
&\textstyle\mathbb{E}_n\qty[\exp(\lambda_{n+1} \mathbbm{1}_{E_{n+1}(\delta)}u_{n+1}^2)] \\ 
\leq&\textstyle \exp\qty((D_{1,n})^+\lambda_n \mathbbm{1}_{E_n(\delta)}u_n^2)\exp(D_{2,n}\mathbbm{1}_{E_n(\delta)})\\
    \leq&\textstyle \exp(\lambda_n \mathbbm{1}_{E_n(\delta)}u_n^2)\exp\left(4 B_{\tref{thm:bound},1}^2 C_{\tref{lem:inductive step},1}^2 T_n\right).
\end{align}
For $n \geq \bar{n}$, define $Z_n \doteq \exp(\lambda_n \mathbbm{1}_{E_n(\delta)}u_n^2 - 4 B_{\tref{thm:bound},1}^2 C_{\tref{lem:inductive step},1}^2\sum_{i=\bar{n}}^{n-1}T_i)$. 
We next show that $\qty{Z_n}$ is a supermartingale with respect to the filtration $\{\mathcal{F}_n\}$. 
Indeed, it holds that
\begin{align}
  \textstyle\E_n\qty[Z_{n+1}] =&\textstyle \mathbb{E}_n\qty[\exp(\lambda_{n+1} \mathbbm{1}_{E_{n+1}(\delta)}u_{n+1}^2)] \exp(- 4 B_{\tref{thm:bound},1}^2 C_{\tref{lem:inductive step},1}^2 \sum_{i=\bar{n}}^{n}T_i) \\
  \leq&\textstyle\exp(\lambda_n \mathbbm{1}_{E_n(\delta)}u_n^2)\exp\left(4 B_{\tref{thm:bound},1}^2 C_{\tref{lem:inductive step},1}^2T_n\right)\exp(- 4 B_{\tref{thm:bound},1}^2 C_{\tref{lem:inductive step},1}^2 \sum_{i=\bar{n}}^{n}T_i) \\
  =&\textstyle Z_n.
\end{align}
For any $\delta' \in (0, 1 - \delta)$, 
we now construct $B(\delta, \delta')$.
By Ville's maximal inequality,
we get for any $\epsilon$,
\begin{align}
  &\textstyle\Pr(\sup_{n\geq \bar{n}} \exp(\lambda_n \mathbbm{1}_{E_n(\delta)}u^2_n - 4 B_{\tref{thm:bound},1}^2 C_{\tref{lem:inductive step},1}^2 \sum_{i=\bar{n}}^{n-1} T_i) \geq \exp(\epsilon)) \\
  \leq&\textstyle \E\qty[{\exp(\lambda_{\bar{n}} \mathbbm{1}_{E_{\bar{n}}(\delta)}u_{\bar{n}}^2 -\epsilon)}] \\ 
  \leq&\textstyle \E\qty[{\exp({\theta}{T^{-1}_{\bar{n}} B_{\bar{n}}(\delta)^{-1}} \mathbbm{1}_{E_{\bar{n}}(\delta)}u_{\bar{n}}^2 -\epsilon)}] \\
  \leq&\textstyle \E\qty[{\exp({\theta}T^{-1}_{\bar{n}} -\epsilon)}] = {\exp({\theta}T^{-1}_{\bar{n}} -\epsilon)}.
\end{align}
Select $\epsilon$ such that $\exp(\theta T_{\bar{n}}^{-1} - \epsilon) = \delta'$,
then it holds, with probability at least $1-\delta'$,
that for all $n \geq \bar{n}$,
\begin{align}
  \textstyle\lambda_n \mathbbm{1}_{E_n(\delta)}u_n^2 - 4 B_{\tref{thm:bound},1}^2 C_{\tref{lem:inductive step},1}^2\sum_{i=\bar{n}}^{n-1} T_i \leq \theta T_{\bar{n}}^{-1} - \ln \delta',
\end{align}
implying
\begin{align}
  \textstyle\lambda_n \mathbbm{1}_{E_n(\delta)}u_n^2 \leq \theta T_{\bar{n}}^{-1} + \ln (1/\delta') + 4 B_{\tref{thm:bound},1}^2 C_{\tref{lem:inductive step},1}^2C_\tref{asp:T_n} \ln (n+n_0).
\end{align}
Next, we complete Lemma~\ref{lem:inductive step} by removing $\mathbbm{1}_{E_n(\delta)}$ in the LHS of the above inequality. 
For simplicity of notation, 
we denote the RHS of the above inequality as $\epsilon(n,\bar{n},\delta')$.
We then have
\begin{align}
&\textstyle\Pr(\lambda_nu_n^2 \leq \epsilon(n, \bar{n},\delta'), \forall n \geq \bar{n}) \\
=&\textstyle \Pr\left(\bigcap_{n=\bar{n}}^{\infty}\{\lambda_nu_n^2 \leq \epsilon(n,\bar{n},\delta')\}\right) \\
\geq&\textstyle \Pr\left(\bigcap_{n=\bar{n}}^{\infty}\{\lambda_n\mathbbm{1}_{E_n(\delta)}u_n^2 \leq \epsilon(n,\bar{n},\delta')\} \cap E_n(\delta)\right).
\end{align}
To proceed, note that for any two events, $A$ and $B$, we have
\begin{align}
\Pr(A \cap B) = 1 - \Pr(A^c \cup B^c) \geq 1 - \Pr(A^c) - \Pr(B^c) = \Pr(A) + \Pr(B) - 1.
\end{align}
Therefore, we have
\begin{align}
&\textstyle\Pr( \lambda_n u_n^2 \leq \epsilon(n,\bar{n},\delta'), \forall n \geq \bar{n}) \\
\geq&\textstyle \Pr\left(\bigcap_{n=\bar{n}}^{\infty}\{\lambda_n\mathbbm{1}_{E_n(\delta)}u_n^2 \leq \epsilon(n,\bar{n},\delta')\} \cap E_n(\delta)\right) \\
\geq&\textstyle \Pr\left(\bigcap_{n=\bar{n}}^{\infty}\{\lambda_n\mathbbm{1}_{E_n(\delta)}u_n^2 \leq \epsilon(n,\bar{n},\delta')\}\right) + \Pr\left(\bigcap_{n=0}^{\infty}E_n(\delta)\right) - 1 \\
=&\textstyle \Pr(\lambda_n\mathbbm{1}_{E_n(\delta)}u_n^2 \leq \epsilon(n,\bar{n},\delta'), \forall n \geq \bar{n}) + \lim_{n\to\infty}\Pr(E_n(\delta)) - 1 \\
\geq&\textstyle (1-\delta') + (1-\delta) - 1 \\
=&\textstyle 1-\delta-\delta'.
\end{align}
Using the definitions of $\epsilon(n,\bar{n},\delta')$ and $\lambda_n$, we arrive at the following result.
For any $\delta' \in (0,1-\delta)$, 
it holds, with probability at least $1-\delta-\delta'$, 
that for any $n \geq \bar{n}$
\begin{align}
\textstyle u_n^2 \leq \frac{T_n\max\qty{B_n(\delta), 1}}{\theta} \qty(\theta T_{\bar{n}}^{-1} + \ln (1/\delta') + 4 B_{\tref{thm:bound},1}^2 C_{\tref{lem:inductive step},1}C_\tref{asp:T_n} \ln (n+n_0)).
\end{align}
Defining $C_{\tref{lem:inductive step}}$ accordingly then completes the proof.
\end{proof}

\subsubsection{Proof of Lemma~\ref{lem:concentration rs}}
\label{proof:concentration rs}
\begin{proof}
For any $\delta \in (0, 1)$,
we define
$B_n(\delta) = C_\tref{lem:inductive basis}'(n+1)^{C_\tref{lem:inductive basis}}\geq 1$. 
From Lemma~\ref{lem:inductive basis},
we conclude that
\begin{align}
  \Pr(u_n^2 \leq B_n(\delta), \forall n \geq \bar{n}) = 1 \geq 1 - \delta.
\end{align}
This lays out the antecedent term in the implication relationship specified in Lemma~\ref{lem:inductive step}.
We define $C_{\tref{lem:concentration rs}} \doteq C_\tref{lem:inductive basis} + 1$. Let $\delta_1,\delta_2,\cdots,\delta_{C_{\tref{lem:concentration rs}}} > 0$ be such that $\sum_{i=1}^{C_{\tref{lem:concentration rs}}} \delta_i \leq 1$. Repeatedly applying Lemma~\ref{lem:inductive step} for $C_{\tref{lem:concentration rs}}$ times, 
we have that 
it holds, with probability at least $1-\sum_{i=1}^{C_{\tref{lem:concentration rs}}} \delta_i$, that for any $n \geq \bar{n}$,
\begin{align}
\textstyle u_n^2 \leq {T_n^{C_{\tref{lem:concentration rs}}} B_n(\delta)}\prod_{i=1}^{C_{\tref{lem:concentration rs}}} C_{\tref{lem:inductive step}}\left[\ln(1/\delta_i) + 1 + \ln (n+n_0) \right].
\end{align}
By choosing $\delta_i = \delta/C_{\tref{lem:concentration rs}}$ for all $i \in \{1,2,\cdots,C_{\tref{lem:concentration rs}}\}$, 
the previous inequality reads
\begin{align}
\textstyle u_n^2 &\textstyle\leq {C_{\tref{lem:inductive step}}^{C_{\tref{lem:concentration rs}}} T_n^{C_{\tref{lem:concentration rs}}} B_n(\delta)}\left[\ln(C_{\tref{lem:concentration rs}}/\delta) + 1 + \ln (n+n_0) \right]^{C_{\tref{lem:concentration rs}}} \\
&\textstyle\leq C_{\tref{lem:inductive step}}^{C_{\tref{lem:concentration rs}}} C_\tref{asp:T_n}^{C_{\tref{lem:concentration rs}}} C_\tref{lem:inductive basis}' (n+n_0)^{-1} \left[\ln(C_{\tref{lem:concentration rs}}/\delta) + 1 + \ln (n+n_0) \right]^{C_{\tref{lem:concentration rs}}}.
\end{align}
Notably, the above inequality holds only for $n \geq \bar{n}$.
For $n < \bar{n}$,
Lemma~\ref{lem:inductive basis} gives a trivial almost sure bound that
  $u_n^2 \leq C_\tref{lem:inductive basis}' (\bar{n}+1)^{C_\tref{lem:inductive basis}}$.
To get a bound for all $n$,
we select a constant $C_{\tref{lem:concentration rs}, 1}$ such that
\begin{align}
  \textstyle C_{\tref{lem:concentration rs}, 1} \min_{n=0, \dots, \bar{n}} C_{\tref{lem:inductive step}}^{C_{\tref{lem:concentration rs}}} C_\tref{asp:T_n}^{C_{\tref{lem:concentration rs}}} C_\tref{lem:inductive basis}' (n+n_0)^{-1} \left[\ln C_{\tref{lem:concentration rs}} + 1 + \ln (n+n_0) \right]^{C_{\tref{lem:concentration rs}}} > 1.
\end{align}
Since $\ln (C_\tref{lem:concentration rs} / \delta) \geq \ln C_\tref{lem:concentration rs}$,
we have
\begin{align}
  \textstyle C_{\tref{lem:concentration rs}, 1}\min_{n = 0, \dots, \bar{n}} C_{\tref{lem:inductive step}}^{C_{\tref{lem:concentration rs}}} C_\tref{asp:T_n}^{C_{\tref{lem:concentration rs}}} C_\tref{lem:inductive basis}' (n+n_0)^{-1} \left[\ln (C_{\tref{lem:concentration rs}} / \delta) + 1 + \ln (n+n_0) \right]^{C_{\tref{lem:concentration rs}}} > 1.
\end{align}
It then holds that for any $n \geq 0$,
\begin{align}
  \label{eq:combining bounds}
   \textstyle u_n^2 \leq C_\tref{lem:inductive basis}' (\bar{n}+1)^{C_\tref{lem:inductive basis}}  C_{\tref{lem:concentration rs}, 1} C_{\tref{lem:inductive step}}^{C_{\tref{lem:concentration rs}}} C_\tref{asp:T_n}^{C_{\tref{lem:concentration rs}}} C_\tref{lem:inductive basis}' (n+n_0)^{-1} \left[\ln(C_{\tref{lem:concentration rs}}/\delta) + 1 + \ln (n+n_0) \right]^{C_{\tref{lem:concentration rs}}}.
\end{align}
Defining $C_{\tref{lem:concentration rs}}'$ accordingly then completes the proof.
\end{proof}

\subsection{Proof of Theorem~\ref{thm:rate general}}
\label{sec:proof rate general}
\begin{proof}
Recall we have derived the following bound in Theorem~\ref{thm:rs general}
\begin{align}
  \E_n\qty[u_{n+1}^2] \leq (1 + d_n) u_n^2 + e_n - 2c_n u_n^2,
\end{align}
where $d_n = 2b_n^2 + (2+ B_\tref{thm:rs general})a_n,  e_n = 2(1+ B_\tref{thm:rs general})^2 b_n^2 + B_\tref{thm:rs general} a_n$.
  We proceed via applying Lemma~\ref{aux:thm2}. In particular, we identify
\begin{equation}
    z_n \leftrightarrow u_n^2, a_n^* \leftrightarrow 2b_n^2 + (2+ B_\tref{thm:rs general})a_n,
     x_n \leftrightarrow 2(1+ B_\tref{thm:rs general})^2 b_n^2 + B_\tref{thm:rs general} a_n, 
     b_n^* \leftrightarrow 2c_n.
\end{equation}

\tb{(A1) of Lemma~\ref{aux:thm2}.} Theorem~\ref{thm:rate general} (A2) gives $\sum_n (a_n + b_n^2) < \infty$ a.s., thus $\sum_n a_n < \infty$ and $\sum_n b_n^2 < \infty$ a.s.; Theorem~\ref{thm:rate general} (A1) gives $\liminf_{n \to \infty} nc_n > \frac{\eta}{2}$, thus $\sum_n c_n =\infty$ a.s., these trivially confirm that Lemma~\ref{aux:thm2} (A1) holds. 

\tb{(A2) of Lemma~\ref{aux:thm2}.} We have $\liminf_{n \to \infty} nc_n > \frac{\eta}{2}$. So when $n$ is sufficiently large, it holds that $2c_n n \geq \eta$. 

\tb{(A3) of Lemma~\ref{aux:thm2}.} The first part is satisfied by Theorem~\ref{thm:rate general} (A2), we now validate the second part.
By Theorem~\ref{thm:rate general} (A1), there exist some constants $n_1$ and $\eta_0>0$, such that for all $n \geq n_1$, 
we have $\frac{\eta_0}{2} \leq nc_n- \frac{\eta}{2}$. That is,
\begin{align}
    &\textstyle\sum_{n=1}^{\infty} \left(b_n^* - \frac{\eta}{n}\right) \\
    =&\textstyle \sum_{n=1}^{\infty} \left(\frac{n \cdot 2 c_n - \eta}{n}\right)\\
    =&\textstyle \sum_{n=1}^{n_1} \left(2c_n - \frac{\eta}{n}\right) + \sum_{n=n_1+1}^{\infty} \left(\frac{n \cdot 2c_n - \eta}{n}\right)\\
    \geq &\textstyle \sum_{n=1}^{n_1} \left(2c_n - \frac{\eta}{n}\right) + \sum_{n=n_1+1}^{\infty} \frac{\eta_0}{n}.
\end{align}
Here the first summation is finite, while the second summation converges to $\infty$. 
Thus Lemma~\ref{aux:thm2} (A3) holds.

Having verified all conditions of Lemma \ref{aux:thm2} for this setup, invoking Lemma \ref{aux:thm2} yields 
\begin{align}
    \textstyle \lim_{n \to \infty} n^{\eta}u_n^2 =\lim_{n \to \infty} n^{\eta}\qty((z_n - B_\tref{thm:rs general})^+)^2 = 0 \qq{a.s.}
\end{align}
Taking square roots then completes the proof. 
\end{proof}

\section{Proof of Section~\ref{sec:sa}}
\label{proof:sec sa}
We begin with a shared part that provides the common groundwork. The proofs of Theorem~\ref{thm:rate sa} and Theorem~\ref{thm:concentration sa} are then completed separately in Sections~\ref{proof:rate sa} and \ref{proof:concentration sa}, respectively. 

To proceed, we construct a secondary time scale adapted to $\qty{\alpha_t}$. Define
\begin{align}
\textstyle T_m = \frac{C_\alpha \ln^{\nu_1}(m+3)}{(m+3)^{\nu_2}},
\label{eq:anchors}
\end{align}
where $(\nu_1,\nu_2)$ are chosen according to the learning-rate regime
\begin{align}
\begin{cases}
 0 < \nu_1 < 1, \nu_2 = 1 & \qq{for Assumption \ref{asp:lr1} with $\nu = 1$,} \\
\nu_1 = 0, \frac{1}{2} < \nu_2 < \frac{\nu}{2-\nu} & \qq{for Assumption \ref{asp:lr1} with $\nu \in (2/3, 1)$,}  \\
\nu_1 = 0, \nu_2 = 1 & \qq{for Assumption \ref{asp:lr2}.}
\end{cases}
\label{eq:nu_cases}
\end{align}
Initialize $t_0 = 0$ and define, for $m=0,1,\dots$,
\begin{align}
  \label{eq def tm}
 \textstyle t_{m+1} \doteq \min\qty{ k \,\Big|\, \sum_{t=t_m}^{k-1} \alpha_t \geq T_m }.
\end{align}
Under this construction, the segment length satisfies
\begin{align}
  \label{eq bar alpha m lower bound}
  \textstyle \bar{\alpha}_m \;\doteq\; \sum_{t=t_m}^{t_{m+1}-1} \alpha_t \;\ge\; T_m,
\end{align}
so the time axis is partitioned into segments of length $\bar{\alpha}_m$, 
anchored at $\qty{t_m}$.
The relationship between the scales $t$ and $m$ is quantified below 
(cf. Lemma~1 and 2 of \citealt{qian2026almost}):
\begin{lemma}
    \label{lem:lr bounds}
    There exist some $C_\text{\ref{lem:lr bounds}}$ and $m_0$ such that 
    for all $m \ge m_0$ and $t \ge t_m$, we have $\alpha_t \le C_\text{\ref{lem:lr bounds}}T_m^2$.
\end{lemma}

\begin{lemma}
    \label{lem:lr bounds 2}
    There exists some $m_0$ such that
    for all $m \ge m_0$, we have $\bar \alpha_m \le 2 T_m$.
\end{lemma}
We now investigate the iterates ${w_t}$ segment by segment. Telescoping~\eqref{eq:sa update} yields
\begin{align}
  \textstyle w_{t_{m+1}} =&\textstyle w_{t_m} + \sum_{t=t_m}^{t_{m+1}-1} \alpha_t H(w_t, Y_{t+1}) \\
  =&\textstyle w_{t_m} + \sum_{t=t_m}^{t_{m+1}-1} \alpha_t (h(w_{t_m}) + H(w_t, Y_{t+1}) - h(w_{t_m})) \\
  =&\textstyle w_{t_m} + \bar \alpha_m h(w_{t_m}) + \underbrace{\textstyle \sum_{t=t_m}^{t_{m+1}-1} \alpha_t (H(w_t, Y_{t+1}) - h(w_{t_m}))}_{s_m}.
\end{align}
We now examine the iterates ${w_t}$ in \eqref{eq:sa update} with respect to the timescale ${t_m}$. 
Specifically, we can rewrite \eqref{eq:sa update} as \eqref{eq:skeleton sa}:
\begin{align}
  \label{eq:skeleton sa}
  \tag{Skeleton SA}
w_{t_{m+1}} = w_{t_m} + \bar{\alpha}_m h(w_{t_m}) + s_m,
\end{align}
where $s_m$ is the new noise and $\qty{\bar{\alpha}_m}$ is the new learning rates. 

We now construct a Lyapunov function $L(w) = \norm{w}^2$ to study the behavior of $\{w_{t_m}\}$. Since $\norm{\cdot}$ is an inner product norm, we obtain
\begin{align}
\label{eq:main l smooth}
\norm{w_{t_{m+1}}}^2 = \norm{w_{t_m}}^2 + 2\left\langle w_{t_m}, \bar{\alpha}_m h(w_{t_m}) + s_m\right\rangle + \norm{\bar{\alpha}_m h(w_{t_m}) + s_m}^2.
\end{align}
To bound the last two terms, we follow \citet{zou2019finite} to introduce an auxiliary Markov chain. For each time step $t$, we define a dedicated auxiliary chain, denoted by $\qty{\tilde Y_k^{(t)}}_{k \geq 0}$. 
The chain $\qty{\tilde Y_k^{(t)}}$ is constructed to be identical to $\qty{Y_k}$ up to time step $k=t-\tau_{\alpha_t}$, after which it evolves independently according to the fixed transition matrix $P_{w_{t-\tau_{\alpha_t}}}$. 
This step is critical for addressing the time-inhomogeneity of $\qty{Y_k}$, as its transition dynamics depend on the iterates $\qty{w_k}$. By freezing the transition dynamics at $P_{w_{t-\tau_{\alpha_t}}}$ after $k=t-\tau_{\alpha_t}$, the auxiliary chain $\qty{\tilde Y_k^{(t)}}$ isolates the effect of time-dependent noise, enabling rigorous control of the coupling between $\qty{w_k}$ and $\qty{Y_k}$. This approach is essential for deriving almost-sure convergence results and concentration bounds, which are more challenging in the time-inhomogeneous setting compared to the time-homogeneous case considered in \citet{qian2026almost}.

Let $\mathcal{F}_t \doteq \sigma(w_0, Y_1, \dots, Y_t)$ be the filtration 
until time $t$ and recall we use $\E_t\qty[\cdot] \doteq \E\qty[\cdot \mid \mathcal{F}_t]$ to denote the conditional expectation given $\mathcal{F}_t$.
We then perform decomposition 
\begin{align}
  \textstyle s_m = \sum_{t=t_m}^{t_{m+1}-1} \alpha_t (H(w_t, Y_{t+1}) - h(w_{t_m})) = s_{1, m} + s_{2, m} + s_{3, m} + s_{4, m},
\end{align}
where
\begin{align}
  s_{1, m} \doteq&\textstyle \sum_{t=t_m}^{t_{m+1}-1} \alpha_t (H(w_t, Y_{t+1}) - H(w_{t_m}, Y_{t+1})), \\
  s_{2, m} \doteq&\textstyle \sum_{t=t_m}^{t_{m+1}-1} \alpha_t \qty(H(w_{t_m}, Y_{t+1}) - H\qty(w_{t_m}, \tilde Y_{t+1}^{(t)})), \\
  s_{3, m} \doteq&\textstyle \sum_{t=t_m}^{t_{m+1}-1} \alpha_t \qty(H\qty(w_{t_m}, \tilde Y_{t+1}^{(t)}) - \E_{t_m}\qty[H\qty(w_{t_m}, \tilde Y_{t+1}^{(t)})]), \\
  s_{4, m} \doteq&\textstyle \sum_{t=t_m}^{t_{m+1}-1} \alpha_t \qty(\E_{t_m}\qty[H\qty(w_{t_m}, \tilde Y_{t+1}^{(t)})] - h(w_{t_m})).
\end{align}
Notably, $\E_{t_m} \qty[s_{3, m}] = 0$.
Taking $\E_{t_m} \qty[\cdot]$ on both sides of~\eqref{eq:main l smooth} then yields
\begin{align}
  \label{eq main smooth ineq}
  \textstyle \E_{t_m} \qty[\norm{w_{t_{m+1}}}^2] = &\textstyle \norm{{w_{t_m}}}^2 + 2\bar \alpha_m \langle w_{t_m}, h(w_{t_m})\rangle + 2\langle w_{t_m}, \E_{t_m}\qty[s_{1, m} + s_{2, m} + s_{4, m}]\rangle\\
  &\textstyle + \E_{t_m} \qty[\norm{\bar \alpha_m h(w_{t_m}) + s_{1, m} + s_{2, m} + s_{3, m} + s_{4, m}}^2].
\end{align}
We now bound the terms of the RHS one by one.

\begin{lemma}
  \label{lem:bound z1}
  There exists some deterministic $C_\tref{lem:bound z1}$ and $m_0$ such that for all $m \geq m_0$, 
  \begin{align}
  \norm{s_{1, m}} \leq T_m^2 C_\tref{lem:bound z1} \qty(\norm{w_{t_m}} + 1).
  \end{align}
\end{lemma}
The proof is in Appendix~\ref{proof:bound z1}.

\begin{lemma}
  \label{lem:bound z2}
  There exist some deterministic $C_\text{\ref{lem:bound z2},1}$, $C_\text{\ref{lem:bound z2},2}$ and $m_0$ such that for all $m \geq m_0$, 
  \begin{align}
  \label{eq:es2}
  \norm{\E_{t_m}\qty[s_{2, m}]} \leq T_m^2 C_\text{\ref{lem:bound z2},1} \qty(\norm{w_{t_m}} + 1),
  \end{align}
  and 
  \begin{align}
  \label{eq:s2}
  \norm{s_{2, m}} \leq T_m C_\text{\ref{lem:bound z2},2} \qty(\norm{w_{t_m}} + 1).
  \end{align}
\end{lemma}
The proof is in Appendix~\ref{proof:bound z2}. 

\begin{lemma}
  \label{lem:bound z3}
  There exists some deterministic $C_\tref{lem:bound z3}$ and $m_0$ such that for all $m \geq m_0$, 
  \begin{align}
  \norm{s_{3, m}} \leq T_m C_\tref{lem:bound z3} \qty(\norm{w_{t_m}} + 1).
  \end{align}
\end{lemma}
The proof is identical to Lemma~6 in \citet{qian2026almost} and thus is omitted.
\begin{lemma}
  \label{lem:bound z4}
  There exists some deterministic $C_\tref{lem:bound z4}$ and $m_0$ such that for all $m \geq m_0$, 
  \begin{align}
  \norm{s_{4, m}} \leq T_m^2 C_\tref{lem:bound z4} \qty(\norm{w_{t_m}} + 1).
  \end{align}
\end{lemma}
The proof of this lemma is non-trivial and is the key to handling the time-inhomogeneous noise within the skeleton framework. 
The core of the proof is a case analysis based on the relationship between the secondary time scale's anchor point, $t_m$, and the starting point of the auxiliary chain, $t-\tau_{\alpha_t}$.
In particular, for the more challenging case where $t_m < t-\tau_{\alpha_t}$, we employ the tower rule for conditional expectations to bound the bias term.
The full proof is in Appendix~\ref{proof:bound z4}. Assembling the above bounds we obtain
\begin{lemma}
  \label{lem:bound all}
  There exist some deterministic $C_{\tref{lem:bound all},1}>0, C_{\tref{lem:bound all},2}$ and $m_0$ such that for all $m \geq m_0$, 
  \begin{align}
    \E_{t_m}\qty[\norm{w_{t_{m+1}}}^2]\leq& (1 -  C_{\tref{lem:bound all},1} T_m )\norm{w_{t_m}}^2 + C_{\tref{lem:bound all},2}T_m.
  \end{align}
\end{lemma}
The proof is in Appendix~\ref{proof:bound all}.

\subsection{Proof of Theorem~\ref{thm:rate sa}}
\label{proof:rate sa}
\begin{proof}
We first establish the following lemma as an intermediate result.
\begin{lemma}
  \label{lem:rate sa0}
  For any $\nu \in (2/3, 1]$ and $\eta \in (0, 2\nu_2 - 1)$, it holds that
\begin{align}
\textstyle \lim_{m\to\infty} m^{\eta/2}\qty(\norm{w_{t_m}}^2 - \frac{\xi }{\alpha })^+ = 0 \qq{a.s.}
\end{align}
\end{lemma}

\begin{proof}
Denote $z_m \doteq \norm{w_{t_m}}^2$, Lemma~\ref{lem:bound all} has verified \eqref{eq:special rs}. Recall that we proceed under Assumption~\ref{asp:lr1}. Theorem~\ref{thm:bound} (A1) holds trivially, and Theorem~\ref{thm:bound} (A2) is verified by Lemma~\ref{lem:bound diff}. 
We proceed to verify (A1) \& (A2) of Theorem~\ref{thm:rate}.
Let the $\eta$ in Theorem~\ref{thm:rate} be any number in $(0, 2\nu_2 - 1)$.

\tb{When $\nu < 1$.} By the definition of $T_m$ in \eqref{eq:anchors}, we have $T_m = \frac{C_\alpha}{(m+3)^{\nu_2}}$ with $\nu_2 < 1$. So when $m$ is sufficiently large, it holds that $mT_m \to \infty$ as $m \to \infty$,
thus $\liminf_{m\to \infty} mT_m > \frac{\eta}{C_{\tref{lem:bound all},1}}$. 
Additionally, from $\eta < 2\nu_2 - 1$ we derive
\begin{equation*}
\textstyle \sum_{m=1}^{\infty}(m+1)^{\eta}T_m^2 = C_\alpha^2 \sum_{m=1}^{\infty}\frac{(m+1)^{\eta}}{(m+3)^{2\nu_2}} < \infty.
\end{equation*}

\tb{When $\nu = 1$.} We have $\nu_2 = 1$ and $\eta \in (0,1)$, hence $T_m = \frac{C_\alpha \ln^{\nu_1}(m+3)}{m+3}$.
Again we have $\liminf_{m\to \infty} mT_m= \infty > \frac{\eta}{C_{\tref{lem:bound all},1}}$. 
Additionally, from $\eta < 1$ we obtain
\begin{align}
\textstyle\sum_{m=1}^{\infty}(m+1)^{\eta}T_m^2 = C_\alpha^2 \sum_{m=1}^{\infty}\frac{(m+1)^{\eta}\ln^{2\nu_1}(m+3)}{(m+3)^2} < \infty.
\end{align}
Therefore, under Assumption~\ref{asp:lr1},
identifying $\alpha  = C_{\tref{lem:bound all},1}$, $\xi  = C_{\tref{lem:bound all},2}$ and $B_{\tref{thm:bound},1}=C_\tref{lem:bound diff}$, invoking Theorem~\ref{thm:rate} then completes the proof. 
\end{proof}
Having established the convergence rate of \eqref{eq:skeleton sa} in Lemma~\ref{lem:rate sa0}, we now complete the proof by mapping the convergence rate back to \eqref{eq:sa update}. 
Recall the $m_0$ in Lemma \ref{lem:lr bounds 2}. Then for all $m \geq m_0$, from Lemma \ref{lem:diff_w}, it is easy to see for any $t \in [t_m, t_{m+1}]$,
\begin{align}
\norm{w_t}^2 \leq 2C_\tref{lem:diff_w}^2(\bar{\alpha}_m^2 + \norm{w_{t_m}}^2) \leq 8C_\tref{lem:diff_w}^2(T_m^2 + \norm{w_{t_m}}^2).
\label{eq:tm_le_t}
\end{align}
Therefore, we have
\begin{align}
\label{eq:trans}
  \textstyle \qty(\norm{w_t}^2- 8C_\tref{lem:diff_w}^2 \frac{\xi }{\alpha })^+ \leq&\textstyle \max\qty(0, 8C_\tref{lem:diff_w}^2(T_m^2 + \norm{w_{t_m}}^2) - 8C_\tref{lem:diff_w}^2 \frac{\xi }{\alpha }) \\
\leq&\textstyle \max\qty(0, 8C_\tref{lem:diff_w}^2T_m^2 + 8C_\tref{lem:diff_w}^2 \qty(\norm{w_{t_m}}^2 - \frac{\xi }{\alpha })^+) \\
=&\textstyle 8C_\tref{lem:diff_w}^2T_m^2 + 8C_\tref{lem:diff_w}^2 \qty(\norm{w_{t_m}}^2 - \frac{\xi }{\alpha })^+ .
\end{align}
For any $\eta \in (0, 2\nu_2 - 1)$, we have
\begin{align}
    \textstyle m^{\eta/2} \sup_{t\in[t_m,t_{m+1}]} \qty(\norm{w_t}^2 - 8C_\tref{lem:diff_w}^2 \frac{\xi }{\alpha })^+ \leq 8C_\tref{lem:diff_w}^2\qty(m^{\eta/2} T_m^2 + m^{\eta/2}\qty(\norm{w_{t_m}}^2 - \frac{\xi }{\alpha })^+).
\end{align}
The definition of $T_m$ in \eqref{eq:anchors} and Lemma~\ref{lem:rate sa0} then confirms that
\begin{align}
\textstyle \lim_{m\rightarrow\infty} m^{\eta/2} \sup_{t\in[t_m,t_{m+1}]} \qty(\norm{w_t}^2 - 8C_\tref{lem:diff_w}^2 \frac{\xi }{\alpha })^+ = 0.
\end{align}
Denote $B_\tref{thm:rate sa} \doteq \sqrt{8C_\tref{lem:diff_w}^2 \frac{\xi}{\alpha}}$. Notice 
\begin{align}
    \textstyle \qty(\norm{w_t}^2 - B_{\tref{thm:rate sa}}^2)^+ =&\textstyle  \max\qty(0,\norm{w_t}^2 - B_\tref{thm:rate sa}^2) \\
    =&\textstyle  \max\qty(0, \qty(\norm{w_t} - B_\tref{thm:rate sa})\qty(\norm{w_t} + B_\tref{thm:rate sa})) \\
    =&\textstyle  \qty(\norm{w_t} + B_\tref{thm:rate sa})\max\qty(0, \norm{w_t} - B_\tref{thm:rate sa}) \\
    \geq&\textstyle B_\tref{thm:rate sa} \qty(\norm{w_t} - B_\tref{thm:rate sa})^+,
\end{align} 
we then obtain
\begin{align}
\textstyle \lim_{m\rightarrow\infty} m^{\eta/2} \sup_{t\in[t_m,t_{m+1}]} \qty(\norm{w_t} - B_\tref{thm:rate sa})^+ = 0.
\end{align}
To finalize the proof, we now convert the rate coefficient from $m$ to $t$.
\tb{When $\nu < 1$.} Following Section A.10 of \citet{qian2026almost}, we can get $m^{1-\nu_2} = \Theta(t_m^{1-\nu})$, where $a_n = \Theta(b_n)$ means for two non-negative sequences $\{a_n\}$ and $\{b_n\}$, we have $a_n = \mathcal{O}(b_n)$ and $b_n = \mathcal{O}(a_n)$. Thus
\begin{align}
0 &\textstyle = \lim_{m\rightarrow\infty} (m - 1)^{\eta/2} \sup_{t\in[t_{m-1},t_m]} \qty(\norm{w_t} - B_\tref{thm:rate sa})^+ \\
&\textstyle = \lim_{m\rightarrow\infty} m^{\eta/2} \sup_{t\in[t_{m-1},t_m]} \qty(\norm{w_t} - B_\tref{thm:rate sa})^+ \\
&\textstyle = \lim_{m\rightarrow\infty} t_m^{\frac{1-\nu}{2(1-\nu_2)}\eta} \sup_{t\in[t_{m-1},t_m]} \qty(\norm{w_t} - B_\tref{thm:rate sa})^+ \\
&\textstyle \geq \lim_{m\rightarrow\infty} \sup_{t\in[t_{m-1},t_m]} t^{\frac{1-\nu}{2(1-\nu_2)}\eta}\qty(\norm{w_t} - B_\tref{thm:rate sa})^+.
\end{align}
We now optimize the choice of $\nu_2 \in (1/2,\nu/(2-\nu))$ and $\eta \in (0,2\nu_2 - 1)$. Notice that \\$\sup_{\nu_2\in(1/2,\nu/(2-\nu)),\eta\in(0,2\nu_2-1)} \frac{1-\nu}{2(1-\nu_2)}\eta = \frac{3}{4}\nu - \frac{1}{2}$. We then conclude that for any $\zeta < \frac{3}{4}\nu - \frac{1}{2}$,
\begin{equation*}
\textstyle \lim_{t\rightarrow\infty} t^\zeta\qty(\norm{w_t} - B_\tref{thm:rate sa})^+ = 0.
\end{equation*}

\tb{When $\nu = 1$.} Again following Section A.10 of \citet{qian2026almost}, we have
\begin{align}
\textstyle \ln(m + 3) \geq (\ln(t_{m+1} + 2))^{\frac{1}{1+\nu_1}} - B_{\tref{thm:rate sa},1},
\end{align}
where $B_{\tref{thm:rate sa},1}>0$ is some constant. Define $B_{\tref{thm:rate sa},2} \doteq \exp\left(-\zeta B_{\tref{thm:rate sa},1}\right)$, we then have
\begin{align}
0 &\textstyle = \lim_{m\rightarrow\infty} (m + 2)^\zeta \sup_{t\in[t_{m-1},t_m]} \qty(\norm{w_t} - B_\tref{thm:rate sa})^+\\
&\textstyle\geq \lim_{m\rightarrow\infty} B_{\tref{thm:rate sa},2} \exp\left(\zeta \ln^{1/(1+\nu_1)} t_m\right) \sup_{t\in[t_{m-1},t_m]} \qty(\norm{w_t} - B_\tref{thm:rate sa})^+ \\
&\textstyle\geq B_{\tref{thm:rate sa},2} \lim_{t\rightarrow\infty} \exp\left(\zeta \ln^{1/(1+\nu_1)} t\right) \qty(\norm{w_t} - B_\tref{thm:rate sa})^+.
\end{align}
Notice that $\qty(\norm{w_t} - B_\tref{thm:rate sa})^+$ is actually the distance from $w_t$ to the ball $\fBB(B_\tref{thm:rate sa})$, i.e. $d\qty(w_t, \fBB(B_\tref{thm:rate sa}))$. This completes the proof.
\end{proof}

\subsection{Proof of Theorem~\ref{thm:concentration sa}}
\label{proof:concentration sa}
\begin{proof}
The proof is analogous to the previous section. Under Assumption~\ref{asp:lr2}, Assumption~\ref{asp:T_n} is satisfied by \eqref{eq:anchors} with $C_\alpha > \bar C_\alpha$.  Recall that all other assumptions are already verified in Appendix~\ref{proof:rate sa}, we can apply Theorem~\ref{thm:concentration rs} to obtain the concentration bound of $\qty{w_{t_m}}$. 

\begin{lemma}
\label{lem:concentration sa0}
Under Assumption~\ref{asp:lr2}, there exist 
some constants $C_\tref{lem:concentration sa0}$ and $C'_\tref{lem:concentration sa0}$ such that 
for any $\delta > 0$, it holds, with probability at least $1 - \delta$, that for all $m \geq 0$,
\begin{align}
    \textstyle\qty(\qty(\norm{w_{t_m}}^2 - \frac{\xi }{\alpha })^+)^2 \leq C'_\tref{lem:concentration sa0}\frac{1}{m+3}\left[\ln\left(\frac{C_\tref{lem:concentration sa0}}{\delta}\right) + 1 + \ln(m+3)\right]^{C_\tref{lem:concentration sa0}}.
\end{align}
\end{lemma}
Now we map this result from $m$ back to $t$.
First recall that we proceed under Assumption \ref{asp:lr2} and have
\begin{align}
\textstyle\alpha_t = \frac{C_\alpha}{(t+3)\ln^\nu(t+3)}, \quad T_m = \frac{C_\alpha}{m+3}.
\end{align}
Applying the bound in Lemma~\ref{lem:concentration sa0} to \eqref{eq:trans}, similar to Appendix~\ref{proof:rate sa}, for any $t \in [t_m, t_{m+1}]$
\begin{align}
\textstyle\qty(\norm{w_t} - B_\tref{thm:rate sa})^+ \leq&\textstyle \frac{1}{B_\tref{thm:rate sa}}\qty(\norm{w_t}^2 - B_\tref{thm:rate sa}^2)^+ \\
\leq&\textstyle \frac{1}{B_\tref{thm:rate sa}}8C_\tref{lem:diff_w}(T_m^2 + \sqrt{C_\tref{lem:concentration sa0}'}\frac{1}{\sqrt{m+3}} [\ln(C_\tref{lem:concentration sa0}/\delta) + 1 + \ln(m + 3)]^{C_\tref{lem:concentration sa0}/2}).
\end{align}
Following Section A.10 of \citet{qian2026almost}, we obtain
$m + 3 \leq (m_0 + 3)\exp\left(\frac{\ln^{1-\nu}(t_{m+1}+1)}{1-\nu}\right)$. 
Additionally, there exists some constant $B_{\tref{thm:concentration sa},1}$ such that $m+2 \geq B_{\tref{thm:concentration sa},1} \exp\left(\frac{\ln^{1-\nu} t_{m+1}}{1-\nu}\right)$.
Then for all $m \geq m_0$ and any $t \in [t_m, t_{m+1}]$, we have, for some proper constants,
\begin{align}
&\textstyle \qty(\norm{w_t} - B_\tref{thm:rate sa})^+ \\
\leq&\textstyle \frac{1}{B_\tref{thm:rate sa}}8C_\tref{lem:diff_w}\left(\frac{C_\alpha^2}{(m+3)^2} + \sqrt{C_\tref{lem:concentration sa0}'}\frac{1}{\sqrt{m+3}} [\ln(C_\tref{lem:concentration sa0}/\delta) + 1 + \ln(m + 3)]^{C_\tref{lem:concentration sa0}/2}\right) \\
\leq&\textstyle B_{\tref{thm:concentration sa},2}\frac{1}{\sqrt{m+3}}\left[\ln(1/\delta) + B_{\tref{thm:concentration sa},3} + \ln(m + 3)\right]^{C_\tref{lem:concentration sa0}/2} \\
\leq&\textstyle B_{\tref{thm:concentration sa},2}\left(B_{\tref{thm:concentration sa},1}\exp\left(\frac{\ln^{1-\nu}(t_{m+1})}{2(1-\nu)}\right)\right)^{-1}\left[\ln(1/\delta) + B_{\tref{thm:concentration sa},3} + \ln\left((m_0 + 3)\exp\left(\frac{\ln^{1-\nu}(t_{m+1})}{1-\nu}\right)\right)\right]^{C_\tref{lem:concentration sa0}/2} \\
\leq&\textstyle B_{\tref{thm:concentration sa},4}\exp\left(\frac{-\ln^{1-\nu}(t+1)}{2(1-\nu)}\right)\left[\ln(1/\delta) + B_{\ref{thm:concentration sa},5} + \frac{\ln^{1-\nu}(t + 1)}{1-\nu}\right]^{C_\tref{lem:concentration sa0}/2}.
\end{align}
For all $t \leq t_{m_0}$, we can again use a trivial almost sure bound from \eqref{eq:tt gronwall}, following a similar argument as in the case of $n < \bar n$ in Lemma~\ref{lem:concentration rs} in Section~\ref{proof:concentration rs}.
Combining the two bounds then completes the proof.
\end{proof}

\subsection{Auxiliary Lemma}
\begin{lemma}
  \label{lem:diff_w}
  There exists some constant $C_\tref{lem:diff_w}$ such that for any $m$ and any $t \in [t_m, t_{m+1}]$, 
  \begin{align}
    \norm{w_t} \leq& C_\text{\ref{lem:diff_w}}(\bar \alpha_m  + \norm{w_{t_m}}), \\
    \norm{w_t - w_{t_m}} \leq& 2T_m C_\text{\ref{lem:diff_w}}(\norm{w_{t_m}} + 1). \label{cor:diff_w}
  \end{align}
  Furthermore, for any $t \geq 0$, it holds that
  \begin{align}
    \textstyle \norm{w_t } \leq \qty(\norm{w_0 } + C_\tref{lem:diff_w}\sum_{\tau=0}^{t-1} \alpha_\tau ) \exp(C_\tref{lem:diff_w} \sum_{\tau=0}^{t-1} \alpha_\tau) 
    \label{eq:tt gronwall}.
  \end{align}
\end{lemma}
\begin{proof}
    Given that we allow time-inhomogeneous transition kernel, we cannot directly apply Lemma~16 from \citet{qian2026almost}. Nevertheless, under Assumption~\ref{asp:lip}, we can establish that 
    \begin{align}
        \norm{H(w,y)}\leq \norm{H(w,y) - H(0,y)} + \norm{H(0,y)} \leq L_h\norm{w-0}+L_h = L_h\norm{w}+L_h.
    \end{align}
    Crucially, this bound is independent of the state $y$. Because the bound holds uniformly, the specific dynamics of the Markov chain (whether time-homogeneous or inhomogeneous) do not influence the subsequent analysis for this lemma. The remainder of the proof therefore proceeds identically to the approach in \citet{qian2026almost}, which relies primarily on an application of the discrete Gronwall's inequality to establish the bounds on the iterates.
\end{proof}
\begin{lemma}
\label{lem:diff_w1}
Let $\alpha_{i,j} \doteq \sum_{k=i}^{j}\alpha_k$. For any $t_2 > t_1$, if $\alpha_{t_1, t_2-1} L_h < 1$, then the following bound holds
\begin{align}
    \textstyle\norm{w_{t_1} - w_{t_2}} \leq \frac{\alpha_{t_1, t_2-1} L_h}{1 - \alpha_{t_1, t_2-1} L_h} (\norm{w_{t_2}} + 1).
\end{align}
\end{lemma}

\begin{proof}
We first establish a uniform bound on the norm of any intermediate iterate $w_i$ for $i \in [t_1, t_2-1]$ in terms of $\norm{w_{t_2}}$. For any $i \in [t_1, t_2-1]$, we can express $w_i$ in terms of the future iterate $w_{t_2}$ as
\begin{align}
    \textstyle w_i = w_{t_2} - \sum_{j=i}^{t_2-1} \alpha_j H(w_j, Y_{j+1}).
\end{align}
By taking the norm and applying the triangle inequality, along with Assumption~\ref{asp:lip}, we have
\begin{align}
    \textstyle \norm{w_i} \leq \norm{w_{t_2}} + \sum_{j=i}^{t_2-1} \alpha_j \norm{H(w_j, Y_{j+1})} \leq \norm{w_{t_2}} + \sum_{j=i}^{t_2-1} \alpha_j L_h (\norm{w_j} + 1).
\end{align}
Denote $M \doteq \max_{k \in [t_1, t_2-1]} \norm{w_k}$. Then for any $j$ in this interval, $\norm{w_j} \leq M$. The inequality above thus becomes
\begin{align}
    \textstyle\norm{w_i} \leq \norm{w_{t_2}} + \left( \sum_{j=t_1}^{t_2-1} \alpha_j \right) L_h (M + 1) = \norm{w_{t_2}} + \alpha_{t_1, t_2-1} L_h (M+1).
\end{align}
Since this holds for all $i \in [t_1, t_2-1]$, it must also hold for the iterate that achieves the maximum norm $M$. Therefore
\begin{align}
     \textstyle M \leq \norm{w_{t_2}} + \alpha_{t_1, t_2-1} L_h (M+1). \\
\end{align}
Given that $\alpha_{t_1, t_2-1} L_h < 1$, we can solve for $M$
\begin{align}
    \textstyle M=\max_{i \in [t_1, t_2-1]} \norm{w_i} \leq \frac{\norm{w_{t_2}} + \alpha_{t_1, t_2-1} L_h}{1 - \alpha_{t_1, t_2-1} L_h}. \label{eq:intermediate_bound} 
\end{align}
We now bound the norm of the difference between $w_{t_1}$ and $w_{t_2}$
\begin{align}
    \norm{w_{t_2} - w_{t_1}} =&\textstyle \norm{\sum_{i=t_1}^{t_2-1} \alpha_i H(w_i, Y_{i+1})} \\
    \leq&\textstyle \sum_{i=t_1}^{t_2-1} \alpha_i L_h (\norm{w_i} + 1)\\
    \leq&\textstyle \sum_{i=t_1}^{t_2-1} \alpha_i L_h \left( \frac{\norm{w_{t_2}} + \alpha_{t_1, t_2-1} L_h}{1 - \alpha_{t_1, t_2-1} L_h} + 1 \right). \explain{By \eqref{eq:intermediate_bound}}\\
    \leq&\textstyle \sum_{i=t_1}^{t_2-1} \alpha_i L_h \left( \frac{\norm{w_{t_2}} + 1}{1 - \alpha_{t_1, t_2-1} L_h} \right) \\
    =&\textstyle \frac{\alpha_{t_1, t_2-1} L_h}{1 - \alpha_{t_1, t_2-1} L_h} (\norm{w_{t_2}} + 1).
\end{align}
\end{proof}

\subsection{Proof of Lemma~\ref{lem:bound z1}}
\label{proof:bound z1}
\begin{proof}
\begin{align}
    \textstyle \norm{s_{1,m}} \leq&\textstyle \sum_{t=t_m}^{t_{m+1}-1} \alpha_t L_h \norm{w_t - w_{t_m}}\\
    \leq&\textstyle \sum_{t=t_m}^{t_{m+1}-1} \alpha_t L_h 2T_m C_\tref{lem:diff_w}(\norm{w_{t_m}} + 1) \explain{Lemma~\ref{lem:diff_w}}\\
    =&\textstyle 2T_m\bar{\alpha}_m L_h C_\tref{lem:diff_w}(\norm{w_{t_m}} + 1)\\
    \leq&\textstyle 4T_m^2 L_h C_\tref{lem:diff_w}(\norm{w_{t_m}} + 1) \explain{Lemma~\ref{lem:lr bounds 2}}.
\end{align}
\end{proof}

\subsection{Proof of Lemma~\ref{lem:bound z2}}
\label{proof:bound z2}
\begin{proof}
Denote $\alpha_{i,j} \doteq \sum_{k=i}^{j}\alpha_k$.
Using Lemma~\ref{lem:lr bounds}, we have 
\begin{align}
    \alpha_{t-\tau_{\alpha_t},t-1} =\fO(\alpha_t\tau_{\alpha_t})
= \fO(\alpha_t\ln t).
\end{align}
Indeed, by \eqref{eq:mixing},
\begin{align}
\label{eq:tau_alpha_t}
    \textstyle
    \tau_{\alpha_t}
    =
    \left\lceil
    \frac{\ln \alpha_t-\ln C_\tref{asp:markov chain}}{\ln \tau}
    \right\rceil
    =
    \begin{cases}
    \fO(\nu\ln t), & \qq{for Assumption~\ref{asp:lr1},} \\
    \fO(\ln t), & \qq{for Assumption~\ref{asp:lr2}.}
    \end{cases}
\end{align}
For sufficiently large $m_0$, we then have $t-\tau_{\alpha_t} \geq 0$ for all $t \in [t_m,t_{m+1}-1]$ and all $m \geq m_0$.
We now record the conditional version of the auxiliary-chain comparison. Applying the same
telescoping argument as in Section B.4 of \citet{liu2025linearq} after conditioning on 
$\mathcal F_{t-\tau_{\alpha_t}}$, there exists some deterministic constant $C_{\tref{lem:bound z2},3}$ such that
for all sufficiently large $t$,
\begin{align}
\label{eq:diff_aux}
    \textstyle 
    \sum_{y'} 
    \abs{
    \Pr(Y_{t+1}=y' \mid \mathcal F_{t-\tau_{\alpha_t}})
    -
    \Pr(\tilde Y_{t+1}^{(t)}=y' \mid \mathcal F_{t-\tau_{\alpha_t}})
    }
    \leq 
    C_{\tref{lem:bound z2},3}\abs{\fY}\alpha_{t-\tau_{\alpha_t},t-1}\ln(t+t_0)
    \quad \text{a.s.}
\end{align}
This is the conditional form that will be used below.

We next establish the scale estimate needed for the main part of the proof. According to Section A.10 and A.14 of \citet{qian2026almost}, we have
\begin{align}
\label{eq:t_m}
    \ln t_{m+1}
    =
    \begin{cases}
    \fO\qty(\ln^{1+\nu_1}m), 
    & \qq{for Assumption~\ref{asp:lr1} with $\nu=1$,} \\
    \frac{1-\nu_2}{1-\nu}\fO(\ln m), 
    & \qq{for Assumption~\ref{asp:lr1} with $\nu\in(2/3,1)$,} \\
    (1-\nu)^{\frac{1}{1-\nu}}\fO\qty(\ln^{\frac{1}{1-\nu}}m), 
    & \qq{for Assumption~\ref{asp:lr2}.}
    \end{cases}
\end{align}
Combining these bounds with Lemma~\ref{lem:lr bounds} gives
\begin{align}
\label{aux:alpha__}
    \max_{t_m \leq t < t_{m+1}}
    \qty{\alpha_{t-\tau_{\alpha_t},t-1}\ln(t+t_0)}
    =
    \fO(T_m).
\end{align}
Define
\begin{align*}
    I_m^- &\doteq \qty{t\in\qty{t_m,\dots,t_{m+1}-1}: t-\tau_{\alpha_t}<t_m},\\
I_m^+ &\doteq \qty{t\in\qty{t_m,\dots,t_{m+1}-1}: t-\tau_{\alpha_t}\geq t_m}.
\end{align*}
Let $\tau_m^\star \doteq \max_{t_m\leq t<t_{m+1}}\tau_{\alpha_t}$.
If $t\in I_m^-$, then $t-t_m<\tau_{\alpha_t}\leq \tau_m^\star$. Hence, using the monotonicity of
$\alpha_t$,
\begin{align}
\label{eq:boundary_mass_s2}
    \sum_{t\in I_m^-}\alpha_t
    \leq
    \sum_{j=0}^{\tau_m^\star}\alpha_{t_m+j}
    \leq
    (\tau_m^\star+1)\alpha_{t_m}
    \leq
    C\alpha_{t_m}\ln(t_{m+1}+t_0)
    =
    \fO(T_m^2).
\end{align}
For Assumption~\ref{asp:lr1} with $\nu\in(2/3,1)$, the last equality uses
$m^{1-\nu_2}=\Theta(t_m^{1-\nu})$ and the strict choice
$\nu_2<\nu/(2-\nu)$; for Assumption~\ref{asp:lr1} with $\nu=1$ and for
Assumption~\ref{asp:lr2}, it follows from the same scale relations in
Section A.10 and A.14 of \citet{qian2026almost}.

We now prove \eqref{eq:es2}. For the boundary indices $t\in I_m^-$, we use only the crude bound
from Assumption~\ref{asp:lip}:
\begin{align}
    &\textstyle \norm{\E_{t_m}\qty[H(w_{t_m},Y_{t+1}) - H(w_{t_m},\tilde Y_{t+1}^{(t)})]} \\
    \leq&\textstyle \E_{t_m}\qty[\norm{H(w_{t_m},Y_{t+1})} + \norm{H(w_{t_m},\tilde Y_{t+1}^{(t)})}]\\
    \leq& 2L_h(\norm{w_{t_m}}+1).
\end{align}
Therefore, by \eqref{eq:boundary_mass_s2},
\begin{align}
\label{eq:s2_boundary_part}
    \textstyle \sum_{t\in I_m^-}
    \alpha_t \norm{\E_{t_m}\qty[H(w_{t_m},Y_{t+1}) - H(w_{t_m},\tilde Y_{t+1}^{(t)})]} \leq \fO(T_m^2)(\norm{w_{t_m}}+1).
\end{align}

For the main indices $t\in I_m^+$, we have $t-\tau_{\alpha_t}\geq t_m$, so $w_{t_m}$ is
$\mathcal F_{t-\tau_{\alpha_t}}$-measurable. By the tower rule and \eqref{eq:diff_aux},
\begin{align}
    &\textstyle
    \norm{
    \E_{t_m}\qty[H(w_{t_m},Y_{t+1}) -
    H(w_{t_m},\tilde Y_{t+1}^{(t)})]}\\
    =&\textstyle \norm{\E_{t_m}\qty[\E_{t-\tau_{\alpha_t}} \qty[H(w_{t_m},Y_{t+1}) - H(w_{t_m},\tilde Y_{t+1}^{(t)})]]} \\
    \leq& \textstyle \E_{t_m}\qty[\norm{
    \E_{t-\tau_{\alpha_t}}\qty[
    H(w_{t_m},Y_{t+1}) - H(w_{t_m},\tilde Y_{t+1}^{(t)})]}]\\
    =&\textstyle \E_{t_m}\qty[\norm{\sum_{y'}H(w_{t_m},y') \qty(\Pr(Y_{t+1}=y'\mid\mathcal F_{t-\tau_{\alpha_t}}) - \Pr(\tilde Y_{t+1}^{(t)}=y'\mid\mathcal F_{t-\tau_{\alpha_t}}))}] \\
    \leq&\textstyle
    L_h(\norm{w_{t_m}}+1) \E_{t_m}\qty[\sum_{y'} \abs{\Pr(Y_{t+1}=y'\mid\mathcal F_{t-\tau_{\alpha_t}}) - \Pr(\tilde Y_{t+1}^{(t)}=y'\mid\mathcal F_{t-\tau_{\alpha_t}})}]\\
    \leq&\textstyle
    L_h(\norm{w_{t_m}}+1)
    C_{\tref{lem:bound z2},3}\abs{\fY}
    \alpha_{t-\tau_{\alpha_t},t-1}\ln(t+t_0).
\end{align}
Consequently,
\begin{align}
\label{eq:s2_main_part}
    &\textstyle
    \sum_{t\in I_m^+}
    \alpha_t
    \norm{
    \E_{t_m}\qty[
    H(w_{t_m},Y_{t+1})
    -
    H(w_{t_m},\tilde Y_{t+1}^{(t)})
    ]}
    \\
    \leq&\textstyle
    L_h(\norm{w_{t_m}}+1)
    C_{\tref{lem:bound z2},3}\abs{\fY}
    \sum_{t\in I_m^+}
    \alpha_t
    \alpha_{t-\tau_{\alpha_t},t-1}\ln(t+t_0)
    \\
    \leq&\textstyle
    L_h(\norm{w_{t_m}}+1)
    C_{\tref{lem:bound z2},3}\abs{\fY}
    \alpha_{t_m,t_{m+1}-1}
    \max_{t_m\leq t<t_{m+1}}
    \qty{\alpha_{t-\tau_{\alpha_t},t-1}\ln(t+t_0)}
    \\
    \leq&\textstyle
    2T_m
    L_h(\norm{w_{t_m}}+1)
    C_{\tref{lem:bound z2},3}\abs{\fY}
    \max_{t_m\leq t<t_{m+1}}
    \qty{\alpha_{t-\tau_{\alpha_t},t-1}\ln(t+t_0)}
    \explain{By Lemma~\ref{lem:lr bounds 2}}
    \\
    =&\textstyle
    \fO(T_m^2)(\norm{w_{t_m}}+1).
    \explain{By \eqref{aux:alpha__}}
\end{align}
Combining \eqref{eq:s2_boundary_part} and \eqref{eq:s2_main_part}, we obtain
\[
\norm{\E_{t_m}\qty[s_{2,m}]}
\leq
T_m^2 C_{\tref{lem:bound z2},1}(\norm{w_{t_m}}+1)
\]
for a sufficiently large deterministic constant $C_{\tref{lem:bound z2},1}$. This proves
\eqref{eq:es2}.

We now prove \eqref{eq:s2}. By Assumption~\ref{asp:lip},
\begin{align}
    \norm{s_{2,m}}
    \leq&\textstyle
    \sum_{t=t_m}^{t_{m+1}-1}
    \alpha_t
    \qty(
    \norm{H(w_{t_m},Y_{t+1})}
    +
    \norm{H(w_{t_m},\tilde Y_{t+1}^{(t)})}
    )
    \\
    \leq&\textstyle
    2L_h(\norm{w_{t_m}}+1)\bar\alpha_m
    \\
    \leq&\textstyle
    4L_hT_m(\norm{w_{t_m}}+1).
    \explain{By Lemma~\ref{lem:lr bounds 2}}
\end{align}
Selecting $C_{\tref{lem:bound z2},2}$ accordingly completes the proof.
\end{proof}

\subsection{Proof of Lemma~\ref{lem:bound z4}}
\label{proof:bound z4}
\begin{proof}
We first decompose $s_{4,m}$ as 
\begin{align}
\norm{s_{4,m}} \leq&\textstyle \sum_{t=t_m}^{t_{m+1}-1} \alpha_t \norm{\E_{t_m}\qty[H\qty(w_{t_m}, \tilde Y_{t+1}^{(t)}) - h\qty(w_{t_m})]} \\
\leq&\textstyle \sum_{t=t_m}^{t_{m+1}-1} \alpha_t \norm{\E_{t_m}\qty[H\qty(w_{t_m}, \tilde Y_{t+1}^{(t)}) - h\qty(w_{t-\tau_{\alpha_t}})]} \\
&+ \sum_{t=t_m}^{t_{m+1}-1} \alpha_t \norm{\E_{t_m}\qty[h\qty(w_{t-\tau_{\alpha_t}})] - h(w_{t_m})}.
\end{align}
We start from the first term by bounding each summand indexed by $t \in [t_m, t_{m+1} - 1]$.
\tb{When $t_m \geq t-\tau_{\alpha_t}$.} We decompose the summand as
\begin{align}
    &\textstyle\norm{\E_{t_m}\qty[H\qty(w_{t_m}, \tilde Y_{t+1}^{(t)}) - h\qty(w_{t-\tau_{\alpha_t}})]}\\
    =&\textstyle \norm{\sum_y P_{w_{t-\tau_{\alpha_t}}}^{\tau_{\alpha_t}+1}\qty(Y_{t-\tau_{\alpha_t}},y)H(w_{t_m}, y) - \sum_y d_{\mathcal{Y}, w_{t-\tau_{\alpha_t}}}(y)H\qty(w_{t-\tau_{\alpha_t}}, y)}\\
    \leq&\textstyle \norm{\sum_y \qty(P_{w_{t-\tau_{\alpha_t}}}^{\tau_{\alpha_t}+1}\qty(Y_{t-\tau_{\alpha_t}},y) - d_{\mathcal{Y}, w_{t-\tau_{\alpha_t}}}(y))H(w_{t_m}, y)} \\
    &+ \norm{\sum_y d_{\mathcal{Y}, w_{t-\tau_{\alpha_t}}}(y)\qty(H(w_{t_m}, y) - H\qty(w_{t-\tau_{\alpha_t}}, y))}. 
\end{align}
To see the first equality, 
we recall that $\E_{t_m}$ is the conditional expectation conditioned on $Y_1, Y_2, \dots, Y_{t_m}$, so $\qty{\tilde Y_k^{(t)}}$ is not adapted to $\fF_{t_m}$ for $k = t - \tau_{\alpha_t} + 1, \dots, t + 1$.
Then
\begin{align}
    &\textstyle\norm{\sum_y \qty(P_{w_{t-\tau_{\alpha_t}}}^{\tau_{\alpha_t}+1}\qty(Y_{t-\tau_{\alpha_t}},y) - d_{\mathcal{Y}, w_{t-\tau_{\alpha_t}}}(y))H(w_{t_m}, y)}\\
    \leq&\textstyle L_h(\norm{w_{t_m}}+1) \cdot \sum_y \abs{P_{w_{t-\tau_{\alpha_t}}}^{\tau_{\alpha_t}+1}\qty(Y_{t-\tau_{\alpha_t}},y) - d_{\mathcal{Y}, w_{t-\tau_{\alpha_t}}}(y)} \explain{Assumption~\ref{asp:lip}}\\
    \leq&\textstyle L_h(\norm{w_{t_m}}+1) \cdot  C_\tref{asp:markov chain}\tau^{\tau_{\alpha_t}+1} \explain{By \eqref{eq:mixing}}\\
    \leq&\textstyle L_h(\norm{w_{t_m}}+1) \cdot \alpha_t \explain{By \eqref{eq:tau}},
\end{align}
and 
\begin{align}
    &\textstyle\norm{\sum_y d_{\mathcal{Y}, w_{t-\tau_{\alpha_t}}}(y)\qty(H\qty(w_{t_m}, y) - H\qty(w_{t-\tau_{\alpha_t}}, y))}\\
    \leq&\textstyle\sum_y d_{\mathcal{Y}, w_{t-\tau_{\alpha_t}}}(y)\norm{H(w_{t_m}, y) - H\qty(w_{t-\tau_{\alpha_t}}, y)} \\
    \leq&\textstyle L_h \norm{w_{t_m} - w_{t-\tau_{\alpha_t}}} \explain{Assumption~\ref{asp:lip}}\\
    \leq&\textstyle L_h\frac{\alpha_{t-\tau_{\alpha_t}, t_m-1}L_h}{1-\alpha_{t-\tau_{\alpha_t}, t_m-1}L_h}(\norm{w_{t_m}} + 1) \explain{Lemma~\ref{lem:diff_w1}}\\
    \leq&\textstyle 2L_h^2\alpha_{t-\tau_{\alpha_t}, t-1}(\norm{w_{t_m}} + 1).
\end{align}
To see the last inequality (we only need it to hold for sufficiently large $t$),
we first notice that $\alpha_{t-\tau_{\alpha_t}, t_m-1} \leq \alpha_{t-\tau_{\alpha_t}, t-1}$ since we are bounding terms for $t \in [t_m, t_{m+1} - 1]$.
We also have $\lim_{t\to\infty} \alpha_{t-\tau_{\alpha_t}, t-1} = 0$.
So for sufficiently large $t$, we have $1-\alpha_{t-\tau_{\alpha_t}, t_m-1}L_h > 1/2$,
which yields the last inequality.
Putting everything together, we obtain
\begin{align}
    &\textstyle\sum_{t=t_m}^{t_{m+1}-1} \alpha_t \norm{\E_{t_m}\qty[H\qty(w_{t_m}, \tilde Y_{t+1}^{(t)}) - h\qty(w_{t-\tau_{\alpha_t}})]}\\
    \leq&\textstyle \sum_{t=t_m}^{t_{m+1}-1} \alpha_t\cdot L_h(\norm{w_{t_m}}+1) \cdot \alpha_t + \sum_{t=t_m}^{t_{m+1}-1} \alpha_t\cdot L_h^2\alpha_{t-\tau_{\alpha_t}, t-1}(\norm{w_{t_m}} + 1)\\
    =&\textstyle L_h(\norm{w_{t_m}}+1)\qty(\sum_{t=t_m}^{t_{m+1}-1} \alpha_t^2 + \sum_{t=t_m}^{t_{m+1}-1} \alpha_t\cdot L_h\alpha_{t-\tau_{\alpha_t}, t-1})\\
    \leq&\textstyle L_h(\norm{w_{t_m}}+1)\qty( \alpha_{t_m}\alpha_{t_m, t_{m+1}-1} +  \alpha_{t_m, t_{m+1}-1} L_h\max_{t\in\qty[t_m, t_{m+1}-1]}\alpha_{t-\tau_{\alpha_t}, t-1})\\
    \leq&\textstyle L_h(\norm{w_{t_m}}+1)\qty( C_\tref{lem:lr bounds}T_m^2 \cdot 2T_m + 2T_m L_h\max_{t\in\qty[t_m, t_{m+1}-1]}\alpha_{t-\tau_{\alpha_t}, t-1}) \explain{By \eqref{eq bar alpha m lower bound}, Lemma~\ref{lem:lr bounds} and Lemma~\ref{lem:lr bounds 2}}\\
    =&\textstyle \fO(T_m^2)(\norm{w_{t_m}}+1). \explain{By \eqref{aux:alpha__}}
\end{align}

\tb{When $t_m < t-\tau_{\alpha_t}$.} Using the tower rule we have
\begin{align}
\textstyle\E_{t_m}\qty[H\qty(w_{t_m}, \tilde Y_{t+1}^{(t)}) - h\qty(w_{t-\tau_{\alpha_t}})] = \E_{t_m}\qty[\mathbb{E}_{t-\tau_{\alpha_t}}\qty[H\qty(w_{t_m}, \tilde Y_{t+1}^{(t)})] - h\qty(w_{t-\tau_{\alpha_t}})].
\end{align}
Consider the inner expectation term
\begin{align}
    &\textstyle\norm{\mathbb{E}_{t-\tau_{\alpha_t}}\qty[H\qty(w_{t_m}, \tilde Y_{t+1}^{(t)})] - h\qty(w_{t-\tau_{\alpha_t}})} \\
    =&\textstyle \norm{\sum_y P_{w_{t -\tau_{\alpha_t}}}^{\tau_{\alpha_t}+1} \qty(Y_{t-\tau_{\alpha_t}}, y)H(w_{t_m}, y) - h\qty(w_{t-\tau_{\alpha_t}})}\\
    =&\textstyle \norm{\sum_y \qty(P_{w_{t-\tau_{\alpha_t}}}^{\tau_{\alpha_t}+1}\qty(Y_{t-\tau_{\alpha_t}}, y) - d_{\mathcal{Y}, w_{t-\tau_{\alpha_t}}}(y))H(w_{t_m}, y)} \\
    &\textstyle+ \norm{\sum_y d_{\mathcal{Y}, w_{t-\tau_{\alpha_t}}}(y)\qty(H(w_{t_m},y) - H\qty(w_{t-\tau_{\alpha_t}}, y))}.
\end{align}
Then
\begin{align}
    &\textstyle\norm{\sum_y \qty(P_{w_{t-\tau_{\alpha_t}}}^{\tau_{\alpha_t}+1}\qty(Y_{t-\tau_{\alpha_t}}, y) - d_{\mathcal{Y}, w_{t-\tau_{\alpha_t}}}(y))H(w_{t_m}, y)}\\
    \leq&\textstyle L_h (\norm{w_{t_m}}+1) \sum_y \abs{P_{w_{t-\tau_{\alpha_t}}}^{\tau_{\alpha_t}+1} \qty(Y_{t-\tau_{\alpha_t}}, y) - d_{\mathcal{Y}, w_{t-\tau_{\alpha_t}}}(y)}\\
    \leq&\textstyle L_h (\norm{w_{t_m}}+1) C_\tref{asp:markov chain} \alpha_t, \explain{By \eqref{eq:mixing} and \eqref{eq:tau}}
\end{align}
and
\begin{align}
    &\textstyle\norm{\sum_y d_{\mathcal{Y}, w_{t-\tau_{\alpha_t}}}(y)\qty(H(w_{t_m},y) - H\qty(w_{t-\tau_{\alpha_t}}, y))}\\
    \leq&\textstyle \sum_y d_{\mathcal{Y}, w_{t-\tau_{\alpha_t}}}(y)\norm{H(w_{t_m},y) - H\qty(w_{t-\tau_{\alpha_t}}, y)}\\
    \leq&\textstyle L_h\norm{w_{t_m}-w_{t-\tau_{\alpha_t}}} \explain{Assumption~\ref{asp:lip}}\\
    \leq&\textstyle L_hC_\tref{lem:diff_w}\cdot 2T_m(\norm{w_{t_m}}+1)\explain{Lemma~\ref{lem:diff_w}}\\
    \leq& T_mL_hC_{\tref{lem:bound z4},1}(\norm{w_{t_m}}+1).
\end{align}
That is, we obtain
\begin{align}
    &\textstyle\sum_{t=t_m}^{t_{m+1}-1} \alpha_t \norm{\E_{t_m}\qty[H\qty(w_{t_m}, \tilde Y_{t+1}^{(t)}) - h\qty(w_{t-\tau_{\alpha_t}})]}\\
    \leq&\textstyle \sum_{t=t_m}^{t_{m+1}-1} \alpha_t\qty(L_h(\norm{w_{t_m}}+1) \cdot C_\tref{asp:markov chain}\alpha_t + T_mL_hC_{\tref{lem:bound z4},1}(\norm{w_{t_m}} + 1))\\
    =&\textstyle ( C_\tref{asp:markov chain}+ C_{\tref{lem:bound z4},1}) L_h(\norm{w_{t_m}}+1)\sum_{t=t_m}^{t_{m+1}-1} \alpha_t\qty(\alpha_t + T_m)\\
    \leq&\textstyle ( C_\tref{asp:markov chain}+ C_{\tref{lem:bound z4},1}) L_h(\norm{w_{t_m}}+1)2T_m\qty(\max_{t\in [t_m,t_{m+1}-1]}\alpha_t + T_m) \explain{By \eqref{eq bar alpha m lower bound} and Lemma~\ref{lem:lr bounds 2}}\\
    \leq&\textstyle C_{\tref{lem:bound z4},2}T_m^2 (\norm{w_{t_m}}+1). \explain{Lemma~\ref{lem:lr bounds 2}}
\end{align}
We now consider the second term. We first establish the Lipschitz continuity of $h(w)$.

\begin{lemma}
\label{lem:lip_h}
Under Assumptions~\ref{asp:markov chain}-\ref{asp:lip}, there exists some deterministic $L_h'$ such that for any $w_1, w_2$
\begin{align}
    \norm{h(w_1) - h(w_2)} \leq L_h' \norm{w_1 - w_2}.
\end{align}
\end{lemma}
\begin{proof}
    Recall $h(w)$ is defined as $h(w) = \sum_{y \in \mathcal{Y}} d_{\mathcal{Y},w}(y) H(w, y)$, we have
\begin{align}
    \textstyle\norm{h(w_1) - h(w_2)} =&\textstyle \norm{ \sum_{y} d_{\mathcal{Y},w_1}(y) H(w_1, y) - \sum_{y} d_{\mathcal{Y},w_2}(y) H(w_2, y) } \\
    =&\textstyle \norm{\sum_{y} d_{\mathcal{Y},w_1}(y) [H(w_1, y) - H(w_2, y)] + \sum_{y} [d_{\mathcal{Y},w_1}(y) - d_{\mathcal{Y},w_2}(y)] H(w_2, y)} \\
    \leq&\textstyle \norm{\sum_{y} d_{\mathcal{Y},w_1}(y) [H(w_1, y) - H(w_2, y)]}\\
    &+ \norm{\sum_{y} [d_{\mathcal{Y},w_1}(y) - d_{\mathcal{Y},w_2}(y)] H(w_2, y)}.
\end{align}
Then we have
\begin{align}
    &\textstyle\norm{\sum_{y} d_{\mathcal{Y},w_1}(y) [H(w_1, y) - H(w_2, y)]} \\
    \leq&\textstyle\sum_{y} d_{\mathcal{Y},w_1}(y) \norm{H(w_1, y) - H(w_2, y)}  \explain{Assumption~\ref{asp:lip}}\\
    \leq&\textstyle L_h \norm{w_1 - w_2}, \explain{Assumption~\ref{asp:lip}}
\end{align}
and
\begin{align}
    &\textstyle\norm{\sum_{y} [d_{\mathcal{Y},w_1}(y) - d_{\mathcal{Y},w_2}(y)] H(w_2, y)} 
    \leq\textstyle \norm{d_{\mathcal{Y},w_1} - d_{\mathcal{Y},w_2}}_1 \max_y \norm{H(w_2, y)},
\end{align}
where $\norm{\cdot}_1$ is the $l_1$ norm. By Lemma 5 of \citet{liu2025linearq}, there exists some constant $C_d$ such that $\norm{d_{\mathcal{Y},w_1} - d_{\mathcal{Y},w_2}}_1 \leq \frac{C_d}{1+\norm{w_1}+\norm{w_2}} \norm{w_1 - w_2}$.
Combining with Assumption~\ref{asp:P_lip'}, we get
\begin{align}
     \textstyle\norm{\sum_{y} [d_{\mathcal{Y},w_1}(y) - d_{\mathcal{Y},w_2}(y)] H(w_2, y)}\leq&\textstyle \qty( \frac{C_d}{1+\norm{w_1}+\norm{w_2}} \norm{w_1 - w_2} ) \cdot L_h(\norm{w_2}+1) \\
    \leq&\textstyle C_d L_h \norm{w_1 - w_2}.
\end{align}
Putting the bounds together we get
\begin{align}
    \textstyle\norm{h(w_1) - h(w_2)} \leq L_h \norm{w_1 - w_2} + C_d L_h \norm{w_1 - w_2} =L_h' \norm{w_1 - w_2},
\end{align}
where $L'_h \doteq L_h(1 + C_d)$.
\end{proof}
With this lemma, we can derive
\begin{align}
    \textstyle\sum_{t=t_m}^{t_{m+1}-1} \alpha_t \norm{\E_{t_m} \qty[h\qty(w_{t-\tau_{\alpha_t}})] - h(w_{t_m})}
    =&\textstyle\sum_{t=t_m}^{t_{m+1}-1} \alpha_t \norm{\E_{t_m}\qty[h\qty(w_{t-\tau_{\alpha_t}}) - h(w_{t_m})]}\\
    \leq&\textstyle\sum_{t=t_m}^{t_{m+1}-1} \alpha_t  L_h'\E_{t_m}\qty[\norm{w_{t-\tau_{\alpha_t}}-w_{t_m}}]. \explain{Lemma~\ref{lem:lip_h}}
\end{align}

\tb{When $t_m \geq t-\tau_{\alpha_t}$.}
\begin{align}
    &\textstyle\sum_{t=t_m}^{t_{m+1}-1} \alpha_t \norm{\E_{t_m}\qty[h\qty(w_{t-\tau_{\alpha_t}})] - h(w_{t_m})}\\
    \leq&\textstyle\sum_{t=t_m}^{t_{m+1}-1} \alpha_t  L_h'\E_{t_m}\qty[\norm{w_{t-\tau_{\alpha_t}}-w_{t_m}}]\\
    \leq&\textstyle \sum_{t=t_m}^{t_{m+1}-1} \alpha_t L_h'\frac{\alpha_{t-\tau_{\alpha_t}, t_m-1}L_h}{1-\alpha_{t-\tau_{\alpha_t}, t_m-1}L_h}(\norm{w_{t_m}} + 1)\explain{Lemma~\ref{lem:diff_w1}}\\
    \leq&\textstyle 2L_hL_h'(\norm{w_{t_m}} + 1) \cdot \sum_{t=t_m}^{t_{m+1}-1} \alpha_t\max_{t\in[t_m,t_{m+1}-1]} \alpha_{t-\tau_{\alpha_t}, t_m-1}\\
    \leq&\textstyle 2L_hL_h'(\norm{w_{t_m}} + 1) \cdot 2T_m \cdot o(T_m) \explain{By \eqref{eq bar alpha m lower bound}, Lemma~\ref{lem:lr bounds 2} and \eqref{aux:alpha__}}\\
    \leq&\textstyle C_{\tref{lem:bound z4}, 3}T_m^2 (\norm{w_{t_m}} + 1),
\end{align}
where the second inequality holds because for sufficiently large $t$, we have $1-\alpha_{t-\tau_{\alpha_t}, t_m-1}L_h > \frac{1}{2}$.

\tb{When $t_m < t-\tau_{\alpha_t}$.}
\begin{align}
    &\textstyle\sum_{t=t_m}^{t_{m+1}-1} \alpha_t \norm{\E_{t_m}\qty[h\qty(w_{t-\tau_{\alpha_t}})] - h(w_{t_m})}\\
    \leq&\textstyle\sum_{t=t_m}^{t_{m+1}-1} \alpha_t  L_h'\E_{t_m}\qty[\norm{w_{t-\tau_{\alpha_t}}-w_{t_m}}]\\
    \leq&\textstyle \sum_{t=t_m}^{t_{m+1}-1} \alpha_t L_h'C_\tref{lem:diff_w}\cdot2T_m(\norm{w_{t_m}} + 1)\explain{Lemma~\ref{lem:diff_w}}\\
    \leq&\textstyle 2T_m L_h'C_\tref{lem:diff_w}\cdot2T_m(\norm{w_{t_m}} + 1) \explain{By \eqref{eq bar alpha m lower bound} and Lemma~\ref{lem:lr bounds 2}}\\
    \leq&\textstyle T_m^2 C_{\tref{lem:bound z4}, 4}(\norm{w_{t_m}} + 1).
\end{align}
Combining all the bounds then completes the proof.
\end{proof}

\subsection{Proof of Lemma~\ref{lem:bound all}}
\label{proof:bound all}
\begin{proof}
We first bound the last two terms in~\eqref{eq main smooth ineq}.
For the first of the two, we have
\begin{align}
    &\langle w_{t_m}, \E_{t_m}[s_{1, m} + s_{2, m} + s_{4, m}]\rangle \\
    \leq&\norm{w_{t_m}} \norm{\E_{t_m}\qty[s_{1, m} + s_{2, m} + s_{4, m}]}  \\
    \leq& \norm{w_{t_m}} \qty(\E_{t_m}\qty[\norm{s_{1, m}} + \norm{s_{4, m}}] + \norm{\E_{t_m} s_{2, m}}) \\
    \leq& \norm{w_{t_m}} T_m^2 (C_\tref{lem:bound z1}  + C_\tref{lem:bound z4} + C_{\tref{lem:bound z2},1} )(\norm{w_{t_m}} + 1) \\
    =& \fO\left(T_m^2(\norm{w_{t_m}}^2 + 1)\right).
\end{align}
For the second of the two,
we notice from Assumption~\ref{asp:lip} that $\norm{h(0)}\leq L_h < L_h'$, Lemma~\ref{lem:lip_h} then gives
\begin{align}
    \norm{h(w_{t_m})} \leq L_h' (\norm{w_{t_m}} + 1).
\end{align}
It then follows easily that 
\begin{align}
    \norm{\bar \alpha_m h(w_{t_m})} = \mathcal{O}(T_m (\norm{w_{t_m}} + 1)).
\end{align}
We then have
\begin{align}
    \norm{\bar \alpha_m h(w_{t_m}) + s_{1, m} + s_{2, m} + s_{3, m} + s_{4, m}}^2 = \mathcal{O}\left(T_m^2(\norm{w_{t_m}}^2 + 1)\right).
\end{align}
We are now ready to refine~\eqref{eq main smooth ineq} as
\begin{align}
    &\textstyle\E_{t_m}\qty[\norm{w_{t_{m+1}}}^2] \\
    \leq&\textstyle (1 - 2\bar \alpha_m C_{\tref{asp:bound inner},1})\norm{w_{t_m}}^2 + 2\bar \alpha_m C_{\tref{asp:bound inner},2} + \fO\qty(T_m^2\qty(\norm{w_{t_m}}^2 + 1))  \\
    \leq&\textstyle (1 - T_m C_{\tref{asp:bound inner},1})\norm{w_{t_m}}^2 + \mathcal{O} \left(T_m \right) + \mathcal{O}\qty(T_m^2(\norm{w_{t_m}}^2 + 1)).
\end{align}
Then for sufficiently large $m$, we have
\begin{align}
\label{eq:recur}
    \textstyle \E_{t_m}\qty[\norm{w_{t_{m+1}}}^2]\leq& (1 -  C_{\tref{lem:bound all},1} T_m )\norm{w_{t_m}}^2 + C_{\tref{lem:bound all},2}T_m.
\end{align}
   
\end{proof}

\vskip 0.2in
\bibliography{bibliography.bib}

\end{document}